\documentclass[runningheads,envcountsame]{llncs}
\usepackage{amsmath, amsxtra, amsfonts, amssymb, amstext}
\usepackage{mathtools}
\usepackage{nicefrac}
\usepackage{xspace}
\usepackage[algo2e,ruled,vlined,linesnumbered]{algorithm2e}
\usepackage{multirow}
\usepackage{footnote}
\usepackage{bm}
\usepackage{bbm}
\usepackage{enumerate}
\usepackage{cite}
\usepackage{thmtools}
\usepackage{thm-restate}
\usepackage{lineno}
\usepackage{pdflscape}
\usepackage{afterpage}
\usepackage[final]{pdfpages}
\usepackage{algorithmicx}
\usepackage{algorithm}
\usepackage{algpseudocode}
\spnewtheorem{myclaim}{Claim}{\bfseries}{\itshape}

\newcommand{\N}{\mathbb{N}}
\newcommand{\R}{\mathbb{R}}

\newcommand{\calS}{\mathcal{S}}

\renewcommand{\epsilon}{\varepsilon}
\newcommand{\eps}{\varepsilon}

\DeclareMathOperator{\E}{\mathbb{E}}

\newcommand{\ooea}{$(1 + 1)$-EA\xspace}

\newcommand{\TwoLins}{\ensuremath{\textsc{TwoLin}^{\rho,\ell}}\xspace}
\newcommand{\TwoLin}{\textsc{TwoLin}\xspace}

\let\oldalign\align
\let\oldendalign\endalign

\renewenvironment{align}
  {\linenomathNonumbers\oldalign}
  {\oldendalign\endlinenomath}

\begin{document}

\title{Two-Dimensional Drift Analysis:}
%\titlenote{Produces the permission block, and copyright information}
\subtitle{Optimizing Two Functions Simultaneously Can Be Hard}
%\subtitlenote{The paper contains only rough proof sketches. We have attached all missing proof in an appendix, in case a reviewer wants to look up any further details. The appendix will be removed for the final version.}

%\author{Johannes Lengler}
%%\authornote{Dr.~Trovato insisted his name be first.}
%%\orcid{1234-5678-9012}
%\affiliation{%
%  \institution{ETH Z{\"u}rich, Switzerland}
%%  \streetaddress{P.O. Box 1212}
%%  \city{Dublin} 
%% \state{Ohio} 
%% \postcode{43017-6221}
%}
\author{Duri Janett and Johannes Lengler}
\authorrunning{D. Janett, J. Lengler}
\institute{Department of Computer Science \\ETH Z{\"u}rich, Switzerland}

\maketitle
\begin{abstract}
In this paper we show how to use drift analysis in the case of two random variables $X_1, X_2$, when the drift is approximatively given by $A\cdot (X_1,X_2)^T$ for a matrix $A$. The non-trivial case is that $X_1$ and $X_2$ impede each other's progress, and we give a full characterization of this case. As application, we develop and analyze a minimal example \TwoLin of a dynamic environment that can be hard. The environment consists of two linear functions $f_1$ and $f_2$ with positive weights, and in each generation selection is based on one of them at random. They only differ in the set of positions that have weight $1$ and $n$. We show that the \ooea with mutation rate $\chi/n$ is efficient for small $\chi$ on \TwoLin, but does not find the shared optimum in polynomial time for large $\chi$.\footnotetext{An extended abstract of this paper, containing only the results without most proofs, has been published at PPSN~\cite{janett2022twoppsn}. Note that the proof of Theorem 2(a) in the PPSN version (Theorem~\ref{thm:2Ddrift}(a) here) contained a mistake in Equation (13), where the expectation was erroneously interchanged with the norm operator. To avoid this issue, the proof of Theorem 3(a) in this paper uses a different (linear) potential function.}
\end{abstract}
% \begin{CCSXML}
%<ccs2012>
%<concept>
%<concept_id>10003752.10010070.10011796</concept_id>
%<concept_desc>Theory of computation~Theory of randomized search heuristics</concept_desc>
%<concept_significance>500</concept_significance>
%</concept>
%</ccs2012>
%\end{CCSXML}
%
%\ccsdesc[500]{Theory of computation~Theory of randomized search heuristics}
%
%%\category{F.2.2}{Theory of Computation}{Analysis of Algorithms and Problem Complexity}[Nonnumerical Algorithms and Problems]
%\keywords{drift analysis, runtime analysis; theory of evolutionary algorithms; dynamic environments}
%%\textbf{Category:} {F.2.2}{Theory of Computation}{Analysis of Algorithms and Problem Complexity}[Nonnumerical Algorithms and Problems]\\
%

\section{Introduction}
\label{sec:Intro}

Evolutionary Algo\-rithms (EAs) and other Randomized Search Heuristics (RSHs) are general-purpose optimization heuristics that are used in a wide range of applications. They are black-box optimizers, which means that, for a given function~$f:\calS \to \R$ to be optimized, they can only access $f$ by querying $f(x)$ for search points $x$ (\emph{individuals}) in the search space $\calS$. The value $f(x)$ is also called the \emph{fitness} of $x$. EAs use an iterative approach in which they maintain a set of search points called the current \emph{population}. In each generation, they create new candidate solutions (\emph{offspring}) from the current population (\emph{parent population}) by alteration and recombination (\emph{mutation} and \emph{crossover}). Then the next population is selected from parents and offspring, based on their fitnesses. Thus EAs are a generalized form of local search. 

The best studied discrete search space is the hypercube $\calS = \{0,1\}^n$. One of the simplest EAs on the hypercube is the \ooea, see also Algorithm~\ref{alg:1}. It comes with a \emph{mutation parameter} $\chi >0$, and maintains a population of size one, i.e., only a single parent individual. Then in each generation it creates an offspring by \emph{standard bit mutation}, which flips each bit of the parent independently with probability $\chi/n$. Hence it flips $\chi$ bits in expectation. It then selects the fitter of parent and offspring for the next generation. 

One of the major tools for analyzing EAs is \emph{drift analysis}~\cite{lengler2020drift}. With this approach, every state of the algorithm is mapped to a single non-negative real number, called \emph{potential}, and the potential is zero if and only if the optimum is found. For the \ooea, the potential can simply be the distance of the current individual from the optimum, but can also be a much more complex function, e.g.~\cite{witt2013tight}. The \emph{drift} of the potential is then the expected change in one generation, and it suffices to know bounds on the drift in order to derive bounds on the expected optimization time. Two particularly frequent cases are that the drift is \emph{additive} (independent of the current potential) or \emph{multiplicative} (proportional to the current potential). Drift analysis works well in many situations, but has an important drawback: it is only applicable if the essence of each population can be captured by a single real value. Although this is possible in surprisingly many situations, there are also limitations to this approach. Especially if the population size is large, or if the function $f$ is complicated, then it is not always possible to characterize a population by a single number.

In this paper, we make a step forward to extending the scope of drift analysis. We show how it can still be applied if the population can be described by a \emph{pair} $X^t = (X_1^t,X_2^t)$ of non-negative real numbers instead of a single real number. Specifically, we consider the case that the drift is approximatively multiplicative, i.e., that $\E[X^t-X^{t+1}] \approx A \cdot X^t$ for some matrix $A= \Big(\begin{matrix} a & b \\ c & d\end{matrix}\Big) \in \R^{2\times 2}$, where we treat $X^t$ as column vector. The drift in this case could also be called ``linear'', but the drift in the analogous one-dimensional case is traditionally called ``multiplicative''. We are mainly interested in the case that $a,d>0$ and $b,c <0$, as we will explain below. 

The main contribution of this paper is that we give a general solution for two-dimensional processes that follow such a multiplicative drift. We show that such a process reaches the target potential $(0,0)$ quickly if and only if the matrix $A$ has two positive eigenvalues (where we omit threshold cases where eigenvalues are zero). The formal statement is in Theorem~\ref{thm:2Ddrift}. 
%The only restriction that we make is that the two random variables $X_L,X_R$ \emph{individually} have positive drift (towards the optimum), but impose a negative drift \emph{on each other}. This is the non-trivial case. For example, if $X_L$ has negative drift by itself, then it is not hard to see that the optimum will not be reached efficiently from a situation where $X_L$ is much larger than $X_R$. On the other hand, if $X_L$ and $X_R$ both have positive drift individually, and $X_L$ contributes positively to the drift of $X_R$, then it is easy to argue that $X_R$ goes to zero regardless of $X_L$, and afterwards $X_L$ also goes to zero. Thus the hard case is covered and settled by Theorem~\ref{thm:2Ddrift}.

Let us briefly discuss the assumptions $a,d>0$ and $b,c <0$. For the sake of this discussion, let us assume that the drift is \emph{exactly} $A\cdot X^t$. In this case, the two random variables $X_1,X_2$ \emph{individually} have positive drift, but impose a negative drift \emph{on each other}.\footnote{We follow the convention that we always call drift towards the optimum \emph{positive}, and drift away from the optimum \emph{negative}. Since we assume $X_1,X_2 \ge 0$ with target potential $(0,0)$, this is the reason for considering the difference $X^{t}-X^{t+1}$ for the drift, not $X^{t+1}-X^t$.} In particular, in this case there are values of $X^t$ such that the drift of $X_1$ is positive, and other values of $X^t$ for which the drift of $X_1$ is negative, and likewise for $X_2$. Other cases are either impossible or degenerate: 
\begin{itemize}
\item If $b=0$ or $c=0$, this means that one of the two processes is independent of the other and can be analyzed separately.
\item It is impossible to have multiplicative drift with $b>0$ since then the state $(X_1,X_2)$ with $X_1 =0$ and $X_2 >0$ would have strictly positive drift in $X_1$, in formula $\E[X_1^{t}-X_1^{t+1} \mid X^t = (X_1,X_2)] = a X_1 +bX_2 = bX_2$. Thus the drift is positive towards $X_1=0$. This is impossible because $X_1$ is already at $0$, and we assumed that $X_1\ge 0$. For the same reason, $c>0$ is impossible.
\item If $b,c <0$ and $a<0$ then the drift of $X_1$ is always negative (away from the optimum) for all values of $X^t$ except for $X^t =(0,0)$. More precisely, the drift of $X_1$ is $-aX_1-bX_2 \le -aX_1$. Hence, $X_1$ has a negative multiplicative drift, regardless of the value of $X_2$. Under reasonable assumptions on the step size of the process, the Negative Drift Theorem implies that $X_1$ does not reach $0$ quickly in such situations~\cite{lengler2020drift}. If $b,c <0$ and $a=0$ then the situation is not much better: the drift of $X_1$ is always non-positive. In such a situation we can not hope for a general result telling us that $X_1$ converges to $0$ quickly. 
\item Analogously to the previous item, $X_2$ does not reach $0$ quickly if $b,c <0$ and $d <0$. As before, the case $b,c <0$ and $d=0$ is not much better. 
\end{itemize}
Hence the only remaining case is $b,c<0$ and $a,d >0$, which is the case that we study in this paper.

Our setting has some similarity with the breakthrough result by Jonathan Rowe on linear multi-objective drift analysis~\cite{rowe2018linear}, one of the most underrated papers of the field in the last years. Our positive result overlaps with the result of~\cite{rowe2018linear}.\footnote{For direct comparison it is important to note that~\cite{rowe2018linear} works with the matrix $I-A$ instead of $A$, where $I$ is the identity matrix.} He gives a sufficient condition for fast convergence that works in arbitrary dimensions: convergence is always fast if the matrix $A$ has an eigenvector which contains only real, positive entries, and which has a corresponding positive real eigenvalue. In two dimensions with $b,c<0$ and $a,d>0$, we show in Lemma~\ref{lem:defgamma} that this is the case if and only if both eigenvalues of $A$ are positive. Hence, for two dimensions the positive criterion in~\cite{rowe2018linear} matches our positive criterion, except that our criterion is more explicit.

The criterion from~\cite{rowe2018linear} is sufficient in arbitrary dimension, and the paper also gives a very nice collection of examples. However, the paper does not contain any negative results, so it is unclear whether this criterion is also necessary in general. We suspect that it is not necessary in higher dimensions. In contrast, our main theorem contains matching positive and negative results. Moreover, while~\cite{rowe2018linear} assumes that the drift is \emph{exactly} given by $A\cdot X^t$ (or lower bounded by that), we show our result even if the drift is only \emph{approximately} given by $A\cdot X^t$. This extension is non-trivial, and indeed, the largest portion of our proof goes into showing that the result is robust under such error terms. Most applications have minor order error terms in the drift (and so does the application in this paper), so we believe that this is quite valuable.

\subsection{The Application: \TwoLin}

We demonstrate that two-dimensional drift analysis is useful by analyzing the \ooea on a new dynamic benchmark \TwoLin. We think that both this benchmark and the runtime results have value of their own, so before describing them we first give some context by discussing related result. This will allow us to explain our motivation for considering this specific setting. \smallskip

To apply EAs effectively, it is important to understand \emph{failure modes} that are specific to some EAs or some parameter settings, so that one can avoid employing an algorithm in inadequate situations. This line of research started with the seminal work of Doerr et al.~\cite{doerr2013mutation}, in which they showed that the $(1+1)$-EA with mutation rate $\chi/n$ is inefficient on some monotone functions\footnote{A function is monotone if the fitness improves whenever we flip a zero-bit into a one-bit.} if $\chi>16$, while it is efficient on all monotone functions if $\chi<1$. The transition happens for constant $\chi$, which is the most reasonable parameter regime. In particular, the failure mode is not related to the trivial problems that occur for extremely large mutation rates, $\chi \gg \log n$, where the algorithm fails to produce neighbours in Hamming distance one\cite{witt2013tight}. Subsequently, the results on how the mutation rate and related parameters affect optimiziation of monotone functions have been refined~\cite{colin2014monotonic,lengler2018drift,lengler2019does} and extended to a large collection of other EAs~\cite{lengler2019general,lengler2021exponential}. To highlight just one result, the $(\mu+1)$-EA with standard mutation rate $1/n$ fails on some monotone functions if the population size $\mu$ is too large (but still constant in $n$)~\cite{lengler2021exponential}. Other algorithm-specific failure modes include
\begin{enumerate}[(i)]
\item non-elitist selection strategies with too small offspring reproductive rate (e.g., comma strategies with small offspring population size)~\cite{jagerskupper2007plus,lehre2010negative,dang2016self,antipov2019efficiency,doerr2021lower,rowe2014choice};
\item elitist algorithms (and also some non-elitist strategies) in certain landscapes with deceptive local optima~\cite{dang2021escaping,dang2021non};  
\item the self-adjusting $(1,\lambda)$-EA if the target success probability is too large~\cite{hevia2021self,kaufmann2023onemax,kaufmann2022self};
\item Min-Max Ant Systems with too few ants~\cite{neumann2010few}; probably the compact Genetic Algorithm and the Univariate Marginal Distribution Algorithms have a similar failure mode for too large step sizes~\cite{sudholt2019choice,lengler2021complex}. 
\end{enumerate}
It is important to note that all aforementioned failure modes are specific to the algorithm, not to the problem. Of course, there are many problems which are intrinsically hard, and where algorithms fail due to the hardness of the problem. However, in the above examples there is a large variety of other RSHs which can solve the problems easily. Except for (ii), the failure modes above even happen on the OneMax problem, which is traditionally the easiest benchmark for RSHs. Since failure modes can occur even in simple situations, it is important to understand them in order to avoid them. 

Unfortunately, some failure modes have been found on benchmarks that are rather technical, in particular in the context of monotone functions. Recently, it was discovered that the same failure modes as for monotone functions can be observed by studying certain~\emph{dynamic environments}, more concretely \emph{Dynamic Linear Functions} and the \emph{Dynamic Binary Value} function DynBV~\cite{lengler2018noisy,lengler2020large,lengler2021runtime}. These environments are very simple, so they allow to study failure modes in greater detail. Crucially, failure in such environments is due to the algorithms, not the problems: they all fall within a general class of problems introduced by Jansen~\cite{jansen2007brittleness} and called \emph{partially-ordered EA} (PO-EA) by Colin, Doerr and Ferey~\cite{colin2014monotonic}, which can be solved efficiently by Random Local Search (RLS) and the $(1+1)$-EA. More precisely, the $(1+1)$-EA with mutation rate $\chi/n$ is known to have optimization time $O(n\log n)$ for $\chi<1$, and $O(n^{3/2})$ for $\chi=1$.\footnote{The statement for $\chi=1$ is contained in~\cite{jansen2007brittleness}, but the proof was wrong. It was later proven in~\cite{colin2014monotonic}.} Hence, they are not intrinsically hard. Both dynamic environments define a set of linear functions with positive weights, and redraw the fitness function in each generation from this set. 

A potential counterargument against these two dynamic environments is that the set of fitness functions is very large. Thus, during optimization, the algorithm may never encounter the same environment twice. In applications, it seems more reasonable that the setup switches between a small set of different environments. Such as a chess engine, which is trained against several, but not arbitrarily varying number of opponents, or a robot, which is trained in a few training environments. Thus, here we propose a \emph{minimal example} of a dynamic environment, \TwoLin, in which EAs may exhibit failure modes. For $0\le \ell \le 1$, we define two functions via
\begin{align}
\begin{split}\label{eq:deff1f2}
f_1(x) := f_1^{\ell}(x) & := \sum\nolimits_{i=1}^{\lfloor \ell n\rfloor} nx_i + \sum\nolimits_{i=\lfloor \ell n\rfloor+1}^n \phantom{n}x_i,\\
f_2(x) := f_2^{\ell}(x) & := \sum\nolimits_{i=1}^{\lfloor \ell n\rfloor} \phantom{n}x_i + \sum\nolimits_{i=\lfloor \ell n\rfloor+1}^n nx_i.
\end{split}
\end{align}
Then for $0\le \rho \le 1$, \TwoLins is the probability distribution over $\{f_1^{\ell},f_2^\ell\}$ that chooses $f_1$ with probability $\rho$ and $f_2$ with probability $1-\rho$. In each generation $t$, a random function $f^t \in \{f_1,f_2\}$ is chosen according to \TwoLins, and selection of the next population is based on this fitness function $f^t$. Note that, similar to monotone functions, $f_1$ and $f_2$ share the global optimum at $(1\ldots 1)$ (which is crucial for benchmarks in dynamic optimization), have no local optima, and flipping a zero-bit into a one-bit always increases the fitness. This is why it falls into the framework of PO-EA, and is hence optimized in time $O(n\log n)$ by the $(1+1)$-EA with mutation rate $c/n$, for any constant $c<1$.

\subsubsection{Results on \TwoLin} 
We show that even for the simple setting of \TwoLins, the (1+1)-EA has a failure mode for mutation rates $\chi/n$ that are too large. For all constant values $\rho,\ell \in (0,1)$, we show that for sufficiently small $\chi$ the algorithm finds the optimum of $\TwoLins$ in time $O(n\log n)$ if started with $o(n)$ zero-bits, but it takes superpolynomial time for large values of $\chi$. For the symmetric case $\rho = \ell = .5$ the threshold between the two regimes is at $\chi_0 \approx 2.557$, which is only slightly larger than the best known thresholds for the $(1+1)$-EA on monotone functions ($\chi_0 \approx 2.13$~\cite{lengler2018drift,lengler2019general}) and for general Dynamic Linear Functions and DynBV ($\chi_0 \approx 1.59$~\cite{lengler2018noisy,lengler2020large,lengler2021runtime}). Thus, we successfully identify a minimal example in which the same failure mode of the \ooea shows as for monotone functions and for the general dynamic settings, and it shows almost as early as in those settings. 

For the symmetric case $\rho=\ell =.5$, it is possible to describe a search point (and thus the state of the algorithm) by the number $Z$ of zero-bits, which allows a fairly standard application of common one-dimensional drift theorems. However, in the asymmetric cases $\rho \neq 1/2$ and/or $\ell\neq 1/2$, the drift is no longer a function of $Z$. 
%The benchmark \TwoLins provides a natural usecase for out two-dimensional drift theorem. However, the situation changes completely when we turn to the cases $\rho \neq 1/2$ and $\ell\neq 1/2$. Here we do not have symmetry, and the drift is no longer a function of $Z$. 
This is inherent to the problem. The state of the algorithm is insufficiently characterized by a single quantity. Instead, a natural characterization of the state needs to specify \emph{two} quantities: the number of zero-bits in the left and right part of the string, denoted by $X_L$ and $X_R$, respectively. Close to the optimum, the drift is approximatively multiplicative, so this yields a natural application for our two-dimensional drift theorem.

%Then we obtain a multiplicative drift \emph{in two dimensions}, i.e., there is a $2\times 2$-matrix $A$ such that the drift of the column vector $X = (X_L,X_R)$ is approximatively $A\cdot X$. This can also be called \emph{linear} drift, but it is traditionally called multiplicative drift in the one-dimensional case. Crucially, $X_L$ and $X_R$ contribute \emph{negatively} to each other, i.e, the non-diagonal entries of $A$ are negative.\footnote{We follow the convention that we always call drift towards the optimum \emph{positive}, and drift away from the optimum \emph{negative}. Since the optimum in our case is at $0$, this means that we consider the difference $X^{t}-X^{t+1}$ for the drift, not vice versa.} In this case, there are values for $X$ for which the distance from the optimum increases in expectation. 
It remains open whether the positive result also holds in full generality when the algorithm starts with $\Omega(n)$ zero-bits. However, we provide an interesting \emph{Domination Lemma} that sheds some light on this question. We call $\Delta_{1,0}$ the drift conditional on flipping one zero-bit in the left part of the bit-string and no zero-bit in the right part, and conversely for $\Delta_{0,1}$. Those two terms dominate the drift close to the optimum. We prove that (throughout the search space, not just close to the optimum), if \emph{both} $\Delta_{1,0}$ and $\Delta_{0,1}$ are positive, then the total drift is also positive. This is enough to remove the starting condition for the symmetric case $\rho = \ell = 1/2$. However, in general, the drift close to the optimum is a weighted sum of those two terms, which may be positive without both terms individually being positive.

\section{Preliminaries and Definitions}
\label{sec:definitions}
Throughout the paper, $\chi>0$ and $\rho, \ell \in [0,1]$ are constants, independent of $n$, and all Landau notation is with respect to $n \to \infty$. We say that an event $\mathcal E_n$ holds \emph{with high probability} or \emph{whp} if $\Pr[\mathcal E_n] \to 1$ for $n\to \infty$. 
%We write $[n] := \{1,\ldots,n\}$. 
The environment \TwoLins is the probability distribution on $\{f_1^{\ell},f_2^{\ell}\}$, which assigns probability $\rho$ and $1- \rho$ to $f_1^{\ell}$ and $f_2^{\ell}$, respectively, where $f_1^{\ell}$ and $f_2^{\ell}$ are given by~\eqref{eq:deff1f2}. The \emph{left} and \emph{right} part of a string $x\in \{0,1\}^n$, denoted by $x_L$ and $x_R$, refers to the first $\lfloor \ell n \rfloor$ bits of $x$ and to the remainder of the string, respectively.
%We abbreviate $\TwoLin := \TwoLin^{.5,.5}$. 
We will consider the \ooea with mutation rate $\chi/n$ for maximization on $\mathcal D := \TwoLins$, which is given in Algorithm~\ref{alg:1}. By $Z(x)$ we denote the number of zero-bits in $x\in\{0,1\}^n$, and throughout the paper we denote $Z^t := Z(x^{t})$, where $x^t$ is the search point in generation $t$ as in Algorithm~\ref{alg:1}. 
\begin{algorithm}
  \caption{The $(1+1)$-EA with mutation rate $\chi/n$ in environment $\mathcal D$.
    \label{alg:1}}
%  \begin{algorithmic}[1]
        Sample $x^{0}$ from $\{0,1\}^n$ uniformly at random (or start with pre-specified $x^0$).\\
        \For{$t=0,1,2,3,\dots$}{
            Draw $f^{t}$ from $\mathcal D$. \\ % \hfill // Draw environment for this round \\
            Create $y^{t}$ by flipping each bit of $x^{t}$ independently with probability $\chi/n$.  \\
            Set $x^{t+1}= \arg \max \{f^{t}(x) \mid x \in \{x^{t},y^{t}\}\}$.  
        }
%   \end{algorithmic}
\end{algorithm}
The \emph{runtime} of an algorithm refers to the number of function evaluations before the algorithm evaluates the shared global maximum of $\TwoLins$ for the first time. For typesetting reasons we will write column vectors in horizontal form in inline text, e.g. $(X_L,X_R)$. We denote by $\|x\| := \max_i\{|x_i|\}$ the $\infty$-norm of a vector $x$, and similarly for matrices.

Close to the optimum, events where two or more zero-bits flip contribute negligibly to the drift. This standard argument is given by the following lemma. 
\begin{lemma}\label{lem:several_flips}
Let $\chi>0$ be a constant, and let $\mathcal E_{i}^t$ be the event that the \ooea with mutation rate $\chi/n$ in step $t$ flips exactly $i$ zero-bits (no restriction on the number of one-bits that are flipped). Let $F^{t}$ be the total number of bits that are flipped in that mutation. Let $Z^{t} = Z(x^t)$ be the number of zero-bits in $x^{t}$. Then, regardless of the fitness function and the selection mechanism, for all $\beta\in [0,1]$,
\begin{align*}
\sum_{i=2}^\infty \E\big[|Z^t-Z^{t+1}| \  \big| \ Z^{t} = \beta n, \mathcal E_i^t\big]\cdot \Pr[\mathcal E_i^t] &   \le \sum_{i=2}^\infty \E[F^{t} \mid Z^{t} = \beta n, \mathcal E_i^t]\cdot \Pr[\mathcal E_i^t] \\
& = O(\beta^2).
\end{align*}
In other words, mutations which flip at least two zero-bits contribute at most $\pm O(\beta^2)$ to the drift of $Z(x^{t})$.
\end{lemma}
\begin{proof}
The first inequality is trivial since the change of $Z(x^{t})$ is bounded by the number of flipped bits. The second step is obtained by a standard computation that can for example~be found in the proof of~\cite[Lemma~2]{lengler2021runtime}.
\end{proof}

\section{Two-Dimensional Multiplicative Drift}

This section contains our main result in methodology, Theorem~\ref{thm:2Ddrift} below. First, we give a lemma which states some basic facts about the matrices that we are interested in. All vectors in this section are column vectors. 

\begin{lemma}\label{lem:defgamma}
Let $a,d > 0$ and $b,c < 0$, and consider the real $2\times 2$-matrix $A= \Big(\begin{matrix} a & b \\ c & d\end{matrix}\Big) \in \R^{2\times 2}$. Then the equation
\begin{align}\label{eq:defgamma}
c\gamma^2 + d\gamma = a\gamma + b 
\end{align}
has a unique positive root $\gamma_0>0$. The vector $e_1 := (\gamma_0,1)$ is an eigenvector of $A$ for the eigenvalue $\lambda_1 := c\gamma_0 + d$. The other eigenvalue is $\lambda_2 := a- c\gamma_0$ and has eigenvector $e_2 := (b,-c\gamma_0)$. The vector $e_1$ has two positive real entries, while the vector $e_2$ has a positive and a negative entry. Moreover, 
\begin{itemize}
\item[\textbullet] if $a\gamma_0+b > 0$, or equivalently $c\gamma_0+d >0$, then $\lambda_2 >  \lambda_1 >0$;
\item[\textbullet] if $a\gamma_0+b < 0$, or equivalently $c\gamma_0+d <0$, then $\lambda_2 >0 > \lambda_1$.
\end{itemize}
The values  $\gamma_0, \lambda_1,\lambda_2$ and the entries of $e_1$ and $e_2$ are analytic functions in $a,b,c,d$, and thus smooth in $a,b,c,d$.
\end{lemma}
\begin{proof}
The roots of Equation~\eqref{eq:defgamma} are the roots of the quadratic polynomial $c\gamma^2 + (d-a)\gamma -b$. Its discriminant $(d-a)^2 +4bc > 0$ is positive, so it has two real roots. Since the product of the roots is $-b/c <0$, by Vieta's rule, they must have different signs. This proves existence and uniqueness of $\gamma_0$.

For the eigenvalues and -vectors, we check:
\begin{align}
Ae_1 = \left(\begin{matrix}a\gamma_0+b \\ c\gamma_0+ d\end{matrix}\right) \stackrel{\eqref{eq:defgamma}}{=}  \left(\begin{matrix}\gamma_0(c\gamma_0+d) \\ c\gamma_0+ d\end{matrix}\right) = \lambda_1 e_1,
\end{align}
and
\begin{align}
Ae_2 = \left(\begin{matrix}ab-bc\gamma_0 \\ bc- cd\gamma_0\end{matrix}\right) \stackrel{\eqref{eq:defgamma}}{=}  \left(\begin{matrix}b(a-c\gamma_0) \\ c(c\gamma_0^2- a\gamma_0)\end{matrix}\right) = \lambda_2 e_2.
\end{align}
Recalling $a,d > 0$ and $b,c < 0$, it is trivial to see that $\lambda_2 >0$ in all cases, $e_1$ has two positive entries, and $e_2$ has a positive and a negative entry. Since $\lambda_1 = c\gamma_0+d$, it has the same sign as $a\gamma_0+b$ by~\eqref{eq:defgamma}. Finally, by ~\eqref{eq:defgamma}, $\gamma_0\lambda_1 = a\gamma_0+ b < a\gamma_0 < a\gamma_0 -c\gamma_0^2 = \gamma_0\lambda_2$, which implies $\lambda_1 < \lambda_2$.

For being analytic, it suffices to observe that $\gamma_0$ can be written in the form $x+\sqrt{y}$, where $x,y$ are analytic functions in the parameters of \eqref{eq:defgamma}, and $y > 0$. Hence, $\gamma_0$ is analytic, and the remaining values are analytic functions in $\gamma_0$ and $a,b,c,d$.\qed
\end{proof}

The following theorem is our main result. It is concerned with the situation that we have two-dimensional state vectors, and the drift in state $x \in \R^2$ is approximatively given by $Ax$. There are a few complications that we need to deal with. Firstly, in our application, the drift does not look \emph{exactly} like that, but only approximatively, up to $(1\pm o(1))$ factors. %\footnote{There are two potential places where a $(1\pm o(1))$ factor could be applied: either to the entries $a,b,c,d$ of the matrix, or to the two entries $ax_1+bx_2$ and $cx_1+dx_2$ of the product vector $Ax$. We always mean the first type. Note that they are not the same. For example, consider the case that $ax_1 + bx_2 = 0$. Then $(1\pm o(1))(ax_1 + bx_2) = 0$, but the expression $(1\pm o(1))ax_1 + (1\pm o(1)) bx_2$ (the one that we consider) does not need to be zero.} 
Secondly, even in the positive case, in our application, we will only compute the drift if the number of one-bits is $o(n)$, since the computations would get much more complicated otherwise. These complications are reflected in the theorem.

In fact, these complications are rather typical. At least in the negative case $a\gamma_0 + b <0$, it is impossible for \emph{any} random process on a finite domain that the drift is given \emph{exactly} by $Ax$ everywhere. This is the same as in one dimension: it is impossible to have a negative drift throughout the whole of a finite domain; it can hold in a subset of the domain, but not everywhere. 

\begin{theorem}[Two-Dimensional Multiplicative Drift]\label{thm:2Ddrift}
Assume that for each $n\in \N$ we have a two-dimensional real Markov chain $(X^t)_{t \ge 0} = (X^t(n))_{t \ge 0}$ on $D \subseteq [0,n] \times [0,n]$, i.e., $X^t = (X_1^t,X_2^t)$ as column vector, where $X_1^t, X_2^t \in [0,n]$. Assume further that there is a (constant) real $2\times 2$-matrix $A= \Big(\begin{matrix} a & b \\ c & d\end{matrix}\Big) \in \R^{2\times 2}$ with $a,d >0$ and $b,c <0$, and that there are constants $\kappa,r > 0$ and a function $0 < \sigma =\sigma(n) \le 1$ with $ \sigma=\omega(\sqrt{\log n/n})$ such that $X^t$ satisfies the following conditions.
\begin{enumerate}
\item[A.] \emph{Two-dimensional linear drift.} For all $t \ge 0$ and all $x$ with $\|x\| \le \sigma n$, the drift at $X^t =x$ is $(1\pm o(1))A\cdot x/n$, by which we mean
\begin{align}\label{eq:o(1)-condition}
\E\left[X^t-X^{t+1} \mid X^t = x\right] = \left(\begin{matrix}(1\pm o(1))a\frac{x_1}{n} + (1\pm o(1))b\frac{x_2}{n} \\ (1\pm o(1))c\frac{x_1}{n} + (1\pm o(1))d\frac{x_2}{n} \end{matrix}\right),
\end{align}
where the $o(1)$ terms are uniform over all $t$ and $x$.
\item[B.] \emph{Tail bound on step size.} For all $i\geq 1$, $t\geq 0$, and for all $x = (x_1, x_2)\in D$,
\begin{align}
\Pr[\|X^t-X^{t+1}\| \geq i \mid X^t = x] \le \frac{\kappa}{(1+r)^i}.
\end{align}
\end{enumerate}
Let $\gamma_0$ be the unique positive root of Equation~\eqref{eq:defgamma}, and let $T$ be the hitting time of $(0,0)$, i.e., the first point in time when $X^T = (0,0)$. 

\begin{enumerate}[(a)]
\item If $a\gamma_0 +b > 0$ and $\|X^0\| = o(\sigma n)$, then $T = O(n\log n)$ with high probability.
\item If $a\gamma_0 +b < 0$ and $\|X^0\| \geq \sigma n$, then $T = e^{\Omega(\sigma^2 n)}$ with high probability.
\end{enumerate}
\end{theorem}
\begin{remark}\label{rem:sigma}
While we have included some complications in the statement of Theorem~\ref{thm:2Ddrift} to make it directly applicable to \TwoLin in Section~\ref{sec:twolin}, we have otherwise sacrificed generality to increase readability. Firstly, we restrict ourselves to the case that the drift is $(1\pm o(1))A\cdot x/n$. The scaling factor $1/n$ is quite typical for applications in EAs, but the machinery would work for other factors as well. Secondly, in many applications, including ours, the factors $x_1/n$ and $x_2/n$ reflect the probability of having any change at all, and \emph{conditional on changing} the drift is of order $\Theta(1)$. In this situation, the Negative Drift Theorem used in the proof can be replaced by stronger versions (see~\cite[Section~2.2]{rowe2014choice}), and we can replace the condition $\sigma = \omega(\sqrt{\log n/n})$ by the weaker condition $\sigma = \omega(1)$ for both parts (a) and (b). Moreover, part (b) then holds with a stronger bound of $e^{\Omega(\sigma n)}$ under the weaker condition $\|X^0\| = \omega(1)$. Finally, the step size condition B can be replaced by other conditions~\cite{kotzing2016concentration,lengler2020drift}.
\end{remark}
\begin{proof}[of Theorem~\ref{thm:2Ddrift}]
We start with some preparations that will be helpful for both (a) and (b). 
By Lemma~\ref{lem:defgamma}, the matrix $A$ has two different eigenvalues $\lambda_1, \lambda_2$ (where $\lambda_1$ may be positive or negative) with eigenvectors $e_1$ and $e_2$, where $e_1$ has two positive entries, while $e_2$ has a positive and a negative entry. In other words, the $2\times 2$-matrix $U$ whose columns are given by $e_1$ and $e_2$ satisfies $U^{-1}AU = \left(\begin{matrix}\lambda_1 & 0 \\ 0 & \lambda_2 \end{matrix}\right)$. 
%While the eigenvalue $\lambda_2$ is always positive, $\lambda_1$ is positive in case (a) and negative in case (b). 
Moreover, the eigenvalues and -vectors depend smoothly on $a,b,c,d$. Hence, changing $a,b,c,d$ by some additive term $\beta$ changes the eigenvalues and -vectors by $O(\beta)$ if $\beta$ is small.\footnote{This is because for any matrix $M$ of norm $1$ we can write $\lambda_1(A+\beta M) = \lambda_1(A) + \beta D{\lambda_1}(A)\cdot M + O(\beta^2)$, where the total differential $D\lambda_1(A)$ has bounded norm, and analogously for the other eigenvalues and -vectors.} In particular, writing $\beta = \delta/n$, for every $\eps > 0$ there exists $\delta >0$ such that every matrix $\tilde A$ with $\|A-\tilde A\| < \delta/n$ is invertible and has two different eigenvalues $\tilde \lambda_1, \tilde \lambda_2$ with eigenvectors $\tilde e_1, \tilde e_2$ respectively, such that
\begin{enumerate}[i)]
\item $|\lambda_1 - \tilde \lambda_1| < \eps/n$ and $|\lambda_2 - \tilde \lambda_2| < \eps/n$.
\item $\|e_1 - \tilde e_1\| < \eps/n$ and $\|e_2 - \tilde e_2\| < \eps/n$.
\item %$|U-\tilde U|_\infty < \eps/(n\max\{|U|_\infty, |\tilde U|_\infty\})$ and 
$\|U^{-1}-\tilde U^{-1}\| < \eps/(n\cdot \max\{\|U\|, \|\tilde U\|\})$, where $\tilde U$ is the matrix with columns $\tilde e_1$ and $\tilde e_2$.
\end{enumerate}
For the last point, note that $\|\tilde U\|$ is uniformly bounded by an absolute constant (depending on $a,b,c,d$) if $\delta/n$ is sufficiently small.

Let $x\in \R^2$, and let $\eta :=  \eta(x) := (\eta_1,\eta_2) = U^{-1}x$ and $\tilde \eta = \tilde \eta(x) := (\tilde \eta_1,\tilde \eta_2) := \tilde U^{-1}x$. In other words, we write $x$ in the basis $\{e_1,e_2\}$ by decomposing $x = \eta_1 e_1 + \eta_2 e_2$, and analogously for the basis $\{\tilde e_1,\tilde e_2\}$. Then we claim that iii) implies
\begin{align}\label{eq:epsdelta}
(1-\tfrac{\eps}{n})\|\eta\| \le \|\tilde \eta\| \le (1+\tfrac{\eps}{n})\|\eta\|.
\end{align} 
To check this, first note that  for the identity matrix $I \in \R^{2\times 2}$,
\begin{align}\label{eq:boundtildeU}
\|\tilde U^{-1} U- I\| \le \|U\| \cdot \|\tilde U^{-1}- U^{-1}\| \le \|U\| \cdot\tfrac{\eps}{n\|U\|} = \tfrac{\eps}{n}.
\end{align}
For any vector $v$, this implies
\begin{align}\label{eq:UUinverse}
\|\tilde U^{-1}U v - v\|  = \|(\tilde U^{-1}U-I)v\| \stackrel{\eqref{eq:boundtildeU}}{\le} \tfrac{\eps}{n}\|v\|.
\end{align} 
With $\tilde \eta = \tilde U^{-1}x = \tilde U^{-1}U \eta$, this implies for $v:= \eta$,
\begin{align}\label{eq:etaandtildeeta}
\|\tilde \eta - \eta\| \le \tfrac{\eps}{n}\|\eta\|,
\end{align}
and the right hand side of~\eqref{eq:epsdelta} follows from $\|\tilde \eta\| \le \|\eta\| + \|\tilde \eta -\eta\|$. Reversing the roles of $\eta$ and $\tilde \eta$, we also have 
$\|\eta\| \le (1+ \eps/n)\|\tilde \eta\|$, and multiplying with $(1-\eps/n)$ yields $(1-\eps/n)\|\eta\| \le (1- (\eps/n)^2)\|\tilde \eta\| \le \|\tilde \eta\|$, which is the left hand side of~\eqref{eq:epsdelta}. Finally, we note for later reference that by an analogous computation,~\eqref{eq:boundtildeU} and~\eqref{eq:UUinverse} also hold with $U$ and $\tilde U$ reversed, so for all vectors $v$,
\begin{align}\label{eq:UUinverse2}
\|U^{-1} \tilde U v- v\|  \le \tfrac{\eps}{n}\|v\|.
\end{align} \smallskip

We are now ready to define the potential function. For a vector $x$, recall that $\eta(x) = (\eta_1(x),\eta_2(x))$ is the corresponding vector in basis $\{e_1,e_2\}$. Then we define the potential of $x$ as $f(x) := \eta_1(x)$. In other words, $f$ is the unique linear function with $f(e_1)=1$ and $f(e_2)= 0$. To convince ourselves this choice makes sense, we will show in the following that
\begin{enumerate}
\item[(P1)] $f(x) > 0$ for all $x\in [0,n]^2 \setminus \{(0,0)\}$;
\item[(P2)] $f((0,0)) =0$;
\item[(P3)] there is a constant $\kappa_1>0$ such that $f(x) \ge \kappa_1 \|x\| $ for all $x\in [0,n]^2$;
\item[(P4)] there is a constant $\kappa_2>0$ such that $f(x) \le \kappa_2 \|x\|$ for all $x\in \R^2$.
\end{enumerate}
Note that (P1) and (P3) \emph{only} hold for $x$ in the first quadrant, i.e., for vectors $x$ that do not have negative entries. This will suffice for our purposes because the random variable $X^t$ is restricted to it. So let $x=(x_1,x_2)$ with $x_1,x_2 \geq 0$. Recall from Lemma~\ref{lem:defgamma} that $e_1 = (e_{11}, e_{12})$ has two positive entries, while $e_2 = (e_{21}, e_{22})$ has a negative entry $e_{21}<0$ and a positive entry $e_{22} >0$. By definition, $\eta$ satisfies the equation $x = \eta_1 e_1 + \eta_2 e_2$. Solving this equation for $\eta_1$ with Cramer's rule yields
\begin{align}
\eta_1 = \frac{e_{22}x_1-e_{21}x_2}{e_{11}e_{22}-e_{12}e_{21}} = \frac{e_{22}x_1+|e_{21}|x_2}{e_{11}e_{22}+e_{12}|e_{21}|}.
\end{align}
In the right expression, all factors and terms are non-negative, and the only factors which may be zero are $x_1$ and $x_2$. This shows (P1) and (P2). It also shows (P3) by setting $\kappa_1 := \min\{e_{22},|e_{21}|\}/(e_{11}e_{22}+e_{12}|e_{21}|) > 0$. We can see from the same formula that (P4) holds for all $x\in \R^2$ with $\kappa_2 := \max\{e_{22},|e_{21}|\}/(e_{11}e_{22}+e_{12}|e_{21}|)$, or simply observe that (P4) just repeats the well-known statement that any finite-dimensional linear function is a bounded operator.

A key property of the potential $f(X^t)$ is that it is a \emph{linear} function in $X^t$. In particular, it commutes with expectations, and we make use of this as follows. Let $x$ be a search point with $\|x\|\le \sigma n$, and assume that $n$ is so large that all $o(1)$ terms in~\eqref{eq:o(1)-condition} are at most $\delta$. Then by Condition A there is some matrix $\tilde A$ with $\|A-\tilde A\| < \delta$ such that the drift at $x$ is $\tilde A x/n$. Using the same notation as above, in particular $\tilde \lambda_1, \tilde \lambda_2, \tilde e_1, \tilde e_2$ for the eigenvalues and eigenvectors of $\tilde A$,  
%notation $e_1,e_2,\tilde e_1, \tilde e_2, \eta, \tilde \eta = (\tilde \eta_1, \tilde \eta_2)$ as before, 
we can rewrite this as
\begin{align}\label{eq:driftineigenbasis}
\begin{split}
\E\left[X^{t+1} \mid X^t = x\right] & = x - \tilde A x/n =  (I- \tilde A/n) \cdot (\tilde \eta_1 \tilde e_1 + \tilde \eta_2 \tilde e_2) \\
& = \big(1-\tfrac{\tilde \lambda_1}{n}\big) \tilde \eta_1 \tilde e_1 + \big(1-\tfrac{\tilde \lambda_2}{n}\big) \tilde \eta_2 \tilde e_2.
\end{split}
\end{align}
Since $f$ is a linear function, which commutes with expectations, we thus obtain
\begin{align}\label{eq:driftofpotential}
\E\left[f(X^{t+1}) \mid X^t = x\right] & = \big(1-\tfrac{\tilde \lambda_1}{n}\big) \tilde \eta_1 f(\tilde e_1) + \big(1-\tfrac{\tilde \lambda_2}{n}\big) \tilde \eta_2 f(\tilde e_2).
\end{align}
We first show that the second summand in~\eqref{eq:driftofpotential} has small absolute value, since this part will be identical for (a) and (b). Note that $\tilde \eta_2$ might be negative. Moreover, $f(\tilde e_2)$ might also be negative, since $\tilde e_2$ is not necessarily in the first quadrant, so (P1) does not apply. We use that $\tilde e_2 = \tilde U U^{-1} e_2 = e_2 - (I-\tilde U U^{-1})e_2$ and that $f$ is linear with $f(e_2)=0$. Since we want to achieve a bound in terms of $\eta_1 = f(x)$, we also use that by~\eqref{eq:epsdelta} we have $|\tilde \eta_2| \le \|\tilde \eta\| \le (1+\eps)\|\eta\|$, and we can continue $\|\eta\| = \|U^{-1}x\| \le \|U^{-1}\|\cdot\|x\| \leq \|U^{-1}\|\cdot \eta_1/\kappa_1$ by (P3). Thus, we can compute
 \begin{align}\label{eq:tildee2bound}
 \begin{split}
\big|(1-\tfrac{\tilde \lambda_2}{n}) \tilde \eta_2 f(\tilde e_2)\big| & = \big(1-\tfrac{\tilde \lambda_2}{n}\big) |\tilde \eta_2| \cdot \big|\big(0- f((I- \tilde U U^{-1})e_2) \big)\big|\\
&  \stackrel{\text{(P4)}}{\le} \big(1-\tfrac{\tilde \lambda_2}{n}\big) |\tilde \eta_2|\cdot \kappa_2 \cdot \|(I- \tilde U U^{-1})e_2\| \\
&  \stackrel{\eqref{eq:boundtildeU}}{\le} \big(1-\tfrac{\tilde \lambda_2}{n}\big) |\tilde \eta_2| \cdot \kappa_2\cdot \tfrac{\eps}{n}\|e_2\| \\
& \le \big(1-\tfrac{\tilde \lambda_2}{n}\big)\kappa_2\cdot (1+\eps) \|U^{-1}\|/\kappa_1\cdot \|e_2\| \cdot \tfrac{\eps}{n}\eta_1.
\end{split}
\end{align}
Thus, we can find a constant $C>0$ such that this term is at most $C\tfrac{\eps}{n}\eta_1$. For the other term in \eqref{eq:driftofpotential}, we finally treat part (a) and (b) of the theorem separately.\smallskip

\noindent \emph{(a).} %We choose $\eps >0$ so small that $(1-\lambda_1/n+\eps/n^2)(1+\eps/n)^2 \le 1-\lambda_1/(2n)$ for all $n\in \N$, and a corresponding $\delta >0$. Recall by i) that $\tilde \lambda_1 > \lambda_1-\eps/n$. 
In this case, both eigenvalues $\lambda_1$ and $\lambda_2$ are positive by Lemma~\ref{lem:defgamma}. Our goal is to bound the first summand in~\eqref{eq:driftofpotential}, and we use that $\tilde e_1 = \tilde U U^{-1} e_1 = e_1 - (I-\tilde U U^{-1})e_1$ and $f(e_1)=1$. Moreover, by~\eqref{eq:etaandtildeeta} and (P3), we have $|\tilde \eta_1 - \eta_1| \le \|\tilde \eta - \eta\| \le \tfrac{\eps}{n}\|\eta\| \le \tfrac{\eps}{n} \|U^{-1}\|\cdot \|x\| \le \tfrac{\eps}{n} \|U^{-1}\|\cdot \eta_1/\kappa_1$. This implies $\tilde \eta_1 \le (1+\tfrac{\eps}{n} \|U^{-1}\|/\kappa_1)\eta_1$. Hence,
 \begin{align}\label{eq:tildee1_bound}
 \begin{split}
\big(1-\tfrac{\tilde \lambda_1}{n}\big) \tilde \eta_1 f(\tilde e_1) & = \big(1-\tfrac{\tilde \lambda_1}{n}\big) \tilde \eta_1 \big(1- f((I- \tilde U U^{-1})e_1) \big) \\
& \stackrel{\eqref{eq:boundtildeU}, (P4)}{\le} \big(1-\tfrac{\tilde \lambda_1}{n}\big) \tilde \eta_1 \big(1+\tfrac{\eps}{n}\kappa_2 \|e_1\|\big) \\
& \le \big(1-\tfrac{\tilde \lambda_1}{n}\big)\left(1+\tfrac{\eps}{n}\kappa_2\|e_1\|\right)\big(1+\tfrac{\eps}{n} \|U^{-1}\|/\kappa_1\big)\eta_1\\
& \stackrel{\text{i)}}{\le} \big(1-\tfrac{\lambda_1 - \eps/n}{n}\big)\left(1+\tfrac{\eps}{n}\kappa_2\|e_1\|\right)\big(1+\tfrac{\eps}{n} \|U^{-1}\|/\kappa_1\big)\eta_1.
\end{split}
\end{align}
Thus, there exists a constant $C'>0$ such that the last term is at most $(1-\lambda_1/n + C'\eps/n)\cdot\eta_1$. In particular, by choosing $\eps >0$ sufficiently small (which we obtain by choosing $\delta$ sufficiently small), we can achieve that this expression is at most $(1-\lambda_1/(2n))\cdot \eta_1$, and that~\eqref{eq:tildee2bound} is at most $\lambda_1/(4n) \cdot \eta_1$. In this case, combining \eqref{eq:driftofpotential}, \eqref{eq:tildee2bound} and \eqref{eq:tildee1_bound}, and plugging in $\eta_1 = f(X^t)$ gives 
\begin{align}\label{eq:driftofpotential1}
\E\left[f(X^{t+1}) \mid X^t = x\right] & \le 1 -\tfrac{\lambda_1}{2n}\eta_1 + \tfrac{\lambda_1}{4n}\eta_1 = \big(1-\tfrac{\lambda_1}{4n}\big) f(X^t),
\end{align}
whenever $\|x\| \le \sigma n$.

Hence, the potential has multiplicative drift towards zero, as long as $\|X^t\| \le \sigma n$. Since $f(X^t) = \Theta(\|X^t\|)$ by (P3) and (P4), there is a constant $\nu >0 $ such that $f(X^t) \le \nu\sigma n$ implies $\|X^t\| \le \sigma n$. Thus,~\eqref{eq:driftofpotential} is applicable whenever $f(X^t) \le \nu\sigma n$. In particular, in the interval $[\nu\sigma n/2,\nu\sigma n]$, the potential has a downwards drift of order at least $\Theta(\sigma)$. The starting potential is below this interval as $f(X^0) = \Theta(\|X^0\|) = o(\sigma n)$, so by the Negative Drift Theorem~\cite[Theorem~2]{oliveto2015improved}, with high probability, the potential does not reach the upper boundary of this interval for at least $e^{\Omega(\sigma^2 n)}= \omega(n \log n)$ steps. Thus, with high probability, the random process remains in the region where \eqref{eq:driftofpotential1} holds, and by the Multiplicative Drift Theorem~\cite{doerr2012multiplicative,lengler2020drift}, with high probability, it reaches the optimum in $O(n \log n)$ steps.

\noindent \emph{(b).} Again, we want to bound the first summand in~\eqref{eq:driftofpotential}, but now we want a lower bound. The crucial difference is that $\lambda_1$ and $\tilde \lambda_1$ are now negative, and we will emphasize this by writing $-\tilde \lambda_1 = |\tilde \lambda_1|$. We use that $\tilde e_1 = \tilde U U^{-1} e_1 = e_1 - (I-\tilde U U^{-1})e_1$ and $f(e_1)=1$ again. As before, by~\eqref{eq:etaandtildeeta} and (P3), we have $|\tilde \eta_1 - \eta_1| \le \|\tilde \eta - \eta\| \le \tfrac{\eps}{n}\|\eta\| \le \tfrac{\eps}{n} \|U^{-1}\|\cdot \|x\| \le \tfrac{\eps}{n} \|U^{-1}\|\cdot f(x)/\kappa_1$. This time, we use it in the form $\tilde \eta_1 \ge (1-\tfrac{\eps}{n} \|U^{-1}\|/\kappa_1)\eta_1$. Hence, the first term in~\eqref{eq:driftofpotential} is
 \begin{align}\label{eq:tildee1_bound2}
 \begin{split}
\big(1+\tfrac{|\tilde \lambda_1|}{n}\big) \tilde \eta_1 f(\tilde e_1) & = \big(1+\tfrac{|\tilde \lambda_1|}{n}\big) \tilde \eta_1 \big(1- f((I- \tilde U U^{-1})e_1) \big) \\
& \ge \big(1+\tfrac{|\tilde \lambda_1|}{n}\big) \tilde \eta_1 \big(1-\tfrac{\eps}{n}\kappa_2\|e_1\|\big) \\
& \ge \big(1+\tfrac{|\tilde \lambda_1|}{n}\big)\left(1-\tfrac{\eps}{n}\|e_1\|\right)\big(1-\tfrac{\eps}{n} \|U^{-1}\|/\kappa_1\big)\eta_1.
\end{split}
\end{align}
Since $|\tilde \lambda_1| > |\lambda_1|-\eps/n$, we can find a constant $C'>0$ such that~\eqref{eq:tildee1_bound} is at least $1 +\tfrac{|\tilde \lambda_1|}{n}\eta_1 - C'\tfrac{\eps}{n}\eta_1$. Hence, if we choose $\eps >0$ sufficiently small and plug in $\eta_1 = f(X^t)$, then~\eqref{eq:driftofpotential} gives us 
\begin{align}\label{eq:driftofpotential2}
\E\left[f(X^{t+1}) \mid X^t = x\right] & \ge 1 +\tfrac{| \lambda_1|}{n}\eta_1 - C'\tfrac{\eps}{n}\eta_1 -C\tfrac{\eps}{n}\eta_1 \ge (1+\tfrac{|\lambda_1|}{2n}) f(X^t).
\end{align}
Thus, we have a negative drift, by which we mean a drift away from the optimum. Let $n$ be so large that all $o(1)$ terms in condition A are at most $\delta$. Then~\eqref{eq:driftofpotential2} holds for all $X^t$ with $X^t \le \sigma n$. Since $f(X^t)= \Theta(\|X^t\|)$ by (P3) and (P4), there is a constant $\nu$ such that a potential of $f(X^t) \le \nu \sigma n$ implies $\|X^t\| \le \sigma n$. Thus,~\eqref{eq:driftofpotential2} is applicable whenever $f(X^t) \le \nu\sigma n$. In particular, in the interval $[\nu\sigma n/2,\nu\sigma n]$, the potential has a negative drift of $\Omega(\sigma)$. By the Negative Drift Theorem~\cite[Theorem~2]{oliveto2015improved}, the time to decrease the potential from $\nu\sigma n$ to $\nu\sigma n/2$ is $e^{\Omega(\sigma^2 n)}$ with high probability, which concludes the proof.\qed
\end{proof}
%
%The following theorem is taken from~\cite{rowe2018linear}. As mentioned before, the theorem was called Linear Multi-Objective Drift Theorem there. We prefer the term ``multi-dimensional'' over ``multi-objective'' since our application is a single-objective setting.
%\begin{theorem}[Linear Multi-Dimensional Drift]\label{thm:multidimensional_drift}
%Let  $X_0,X_1,X_2,\ldots$ be a random sequence from a finite set $\mathcal X$. Let $d_1,\ldots,d_m$ be distance functions with respect to target sets $S_1,\ldots,S_m$ respectively\footnote{This means that $d_i(x) = 0$ if $x\in S_i$, and $d_i(x)>0$ otherwise.}, and suppose $S = \bigcap S_i$ is non-empty. Writing $\mathbf{d} = (d_1(x),\ldots,d_m(x))$ as a column vector, suppose there is a non-negative matrix $A$ such that, for all $x \not \in S$:
%\begin{align*}
%\E[\mathbf{d}(X_{t+1}) \mid X_t = x] \le A\mathbf{d}(x)
%\end{align*}
%component-wise. If $A$ has a left eigenvector $\mathbf v$ which contains only real, positive entries, and which has a corresponding real eigenvalue $0<\lambda <1$, then the expectation of the first hitting time $T$ of the set $S$ (that is, the set of states for which all distances are zero) is bounded above by
%\begin{align*}
%\E[T \mid X_0 = x_0] \le \frac{1 + \log(\mathbf v\cdot \mathbf{d}(x_0)/d^*}{1-\lambda},
%\end{align*}
%where 
%\begin{align*}
%d^* = \min\{\mathbf v\cdot \mathbf{d}(x) \mid x\not \in S\}.
%\end{align*}
%\end{theorem}

\subsection{Interpretation and Generalization}\label{sec:interpretation}
We have given Lemma~\ref{lem:defgamma} and Theorem~\ref{thm:2Ddrift} without much explanation. Especially the term $a\gamma_0+b$ comes out of the blue. While the proofs show that this is indeed the correct term to consider, it does not necessarily give an intuitive understanding of the expression. In this section we will give exactly this: an approach which allows us to interpret the value $\gamma_0$ in a natural way. The discussion will outline an alternative proof that could be used to analyzed a two-dimensional problem more directly, without Theorem~\ref{thm:2Ddrift}. In fact, this is how we discovered Theorem~\ref{thm:2Ddrift} in the first place. 

%in the language of \TwoLin. 
%
%The following approach leads to the same equations in the case of multiplicative (linear) drift, but it is possible to apply it also in situations where the drift is not linear.%, and potentially even to higher dimensions. %, as long as the absolute values of its four components are increasing functions.

Consider the situation of Theorem~\ref{thm:2Ddrift}, i.e., we are interested in the case that two random variables $X_1$ and $X_2$ each have positive drift by themselves, but influence each other negatively, and that the drift is linear in $X_1$ and $X_2$. In the notation of Theorem~\ref{thm:2Ddrift}, this corresponds to the case $a,d> 0$ and $b,c <0$ that is also considered there. As outlined in the introduction, this is a particularly interesting case. It means that when $X_1$ is much larger than $X_2$, the drift coming from $X_1$ dominates the drift coming from $X_2$. In particular, $X_1$ has a drift towards zero in this case, while $X_2$ has a drift away from zero. So in particular, $X_1$ and $X_2$ should approach each other. On the other hand, if $X_2$ is much larger than $X_1$, then the situation is reversed: $X_1$ increases in expectation, and $X_2$ decreases. But again, $X_1$ and $X_2$ should approach each other. That means that, if the ratio $X_1/X_2$ is large then it decreases in expectation, and when $X_1/X_2$ is small then it increases in expectation. In particular, for some value of $X_1/X_2$ the drift of $X_1/X_2$ changes from positive to negative.

This suggests the following ansatz. Let $\gamma := X_1/X_2$. We search for a value of $\gamma$ that is self-stabilizing. To this end, we consider the random variable $Y := X_1 - \gamma X_2$, compute the drift of $Y$ and then choose the value $\gamma_0$ for which the drift of $Y$ at $Y=0$ is zero. This leads precisely to Equation~\eqref{eq:defgamma}. By the considerations above, the drift of $Y$ is decreasing in $Y$. Hence, if the drift at $Y=0$ is zero, then the drift for positive $Y$ points towards zero, and the drift for negative $Y$ \emph{also} points towards zero.

\begin{figure}[ht]
\includegraphics[width=\textwidth]{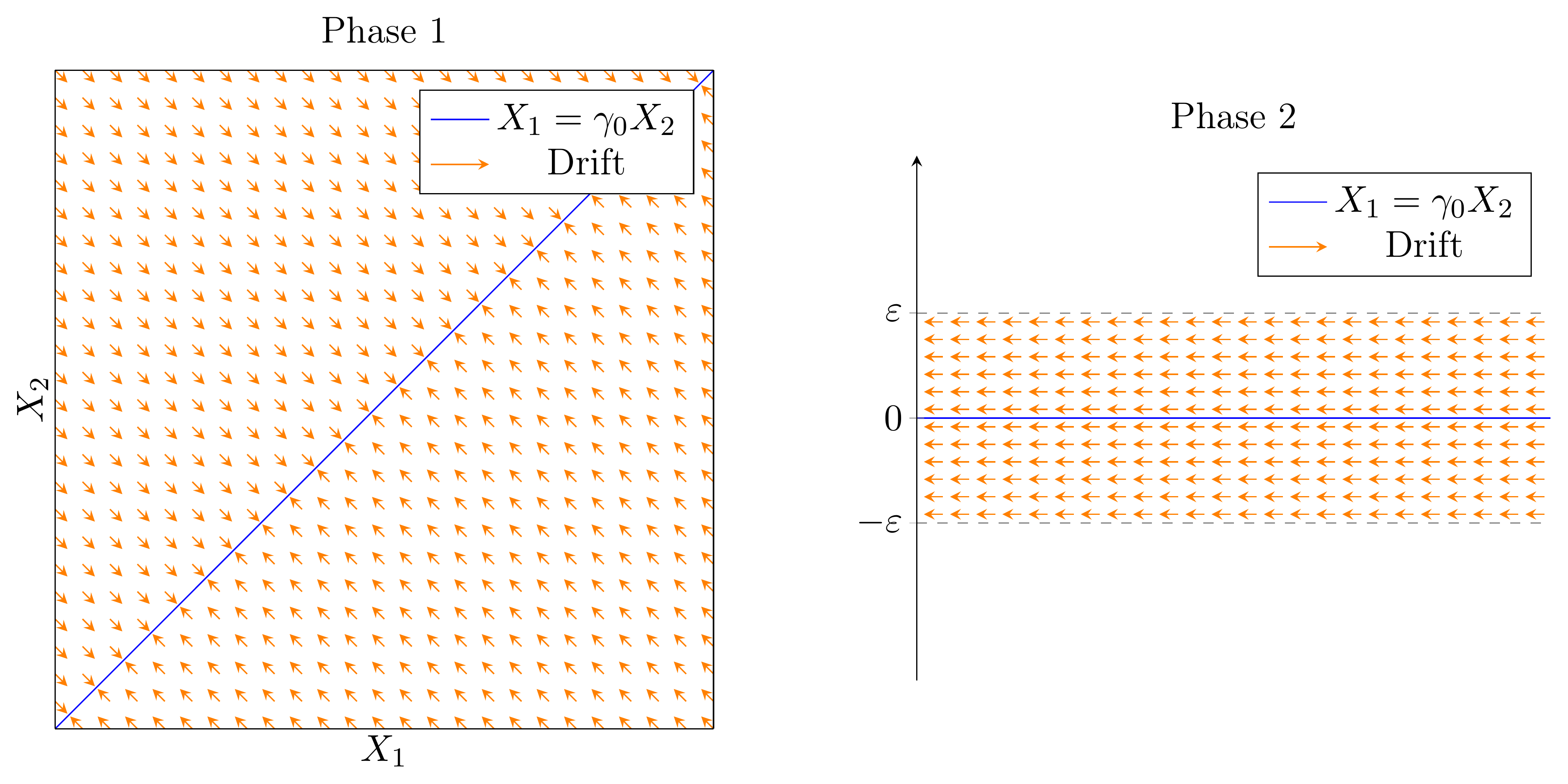}
\caption{Illustration of an alternative proof: in the first phase, the process has a drift towards the subspace defined by the equation $Y = X_1-\gamma_0 X_2 \stackrel{!}{=} 0$. Afterwards, the process stays close to this subspace for a long time, so only the drift close to this subspace is relevant.}\label{fig:phases}
\end{figure}

It is possible to turn this into an analysis as follows, see also Figure~\ref{fig:phases}. We divide the optimization process in two phases. In the first phase, we only consider $Y = X_1-\gamma_0 X_2$. We show that $Y$ approaches $0$ quickly (with the Multiplicative Drift Theorem), and stays very close to zero for a long time (with the Negative Drift Theorem). In the second phase, we analyze the random variable $X := X_1+X_2$. Since we know that $Y$ stays close to zero during this second phase, we do not need to analyze $X$ in the whole search space; it suffices to determine the drift in the subspace where $X_1 = \gamma_0 X_2$. This is a massive restriction, and the drift is reduced to $C\cdot X$ for some constant $C$ under this restriction. Actually, $C= a\gamma_0+b$, which explains why this term appears in Theorem~\ref{thm:2Ddrift}. Thus we can show convergence of $X$ with standard drift arguments, although the details become rather technical.

While this approach is rather cumbersome, it has the potential to work in larger generality than Theorem~\ref{thm:2Ddrift}. In the approach that we have just outlined, we do not actually need that the drift is linear. Essentially, we just need that the drift of $X_1$ is increasing in $X_1$ and decreasing in $X_2$, and vice versa for $X_2$.  This already guarantees a solution $\gamma_0$ (which may then depend on $X$) for the condition that the drift of $Y := X_1 - \gamma X_2$ at $Y=0$ should be zero, and it guarantees that $Y$ then has a drift pushing it towards zero from both directions. Thus there is no fundamental need for linear drift. If the drift is a non-linear function, then we can still define $Y:= X_1-\gamma X_2$ and search for values of $\gamma,X_1,X_2$ such that $Y=0$ and the drift of $Y$ is zero at the same time. Since we have three variables and two equations, this typically cuts out a one-dimensional subspace for $(X_1,X_2)$, though in general not a linear subspace. But it may still be possible to show that in a first phase, the process converges to this one-dimensional subspace, thus effectively reducing the number of dimensions by one. 

Even in dimension larger than two, we can use this argument to reduce the dimension whenever we can find a one-dimensional function $Y$ of $X$ that has drift towards zero from both sides, since then the random process approaches the $Y=0$ subspace. In the linear case of Theorem~\ref{thm:2Ddrift}, such a function $Y$ can be found in terms of eigenvectors: formally, it is the projection onto eigenvectors of negative eigenvalues of the drift matrix. We leave the exploration of non-linear drift and of higher-dimensional linear drift to future work.

%In the case of multiplicative drift, this argument might even work in higher dimension: assume $X$ is a $d$-dimensional random walk with drift $A\cdot X$ for some $d\times d$-matrix $A$, and $v$ is an eigenvector of $A$ with negative eigenvalue $e$. Let us define $Y$ as the projection of $X$ onto $v$, i.e., $Y=(X\cdot v) v$, where $\dot$ is the dot product. Then as before, $Y$ has a drift towards zero from both sides, and we can argue that .  for three variables $X_1,X_2,X_3$, assume that can find $\gamma,\gamma'$ such that $X_1 - \gamma X_2 - \gamma' X_3$ is  In this case, 
%Iterating this argument, it might even be possible to apply the process to dimensions larger than two.

\section{The $(1+1)$-EA on \TwoLins}\label{sec:twolin}

In this section, we analyze the runtime of the $(1+1)$-EA on \TwoLins. The main contribution of this section is Theorem~\ref{thm:main_EA}. This is an application of the methods developed in the previous section. In Section~\ref{sec:sym} below, we will give more precise results for the special case $\rho=\ell=.5$, as well as the Domination Lemma that holds for arbitrary $\rho$ and $\ell$. For readability, we will assume in the following that $\ell n \in \N$. %Throughout this section, the \emph{first part} of a bit string always refers to the first $\ell n$ bits of that string, and the \emph{second part} always refers to the remaining bits. 
We denote by $X_L^t$ and $X_R^t$ the number of zero bits in the left and right part of the $t$-th search point $x^{t}$, respectively, and let $X^t := (X_L^t,X_R^t)$.  It only remains to compute the two-dimensional drift of $X^t$, for which we give the following definition.

\begin{definition}\label{def:driftmain}
Let $\chi >0$ and $\rho,\ell\in [0,1]$. Let $A = \left(\begin{matrix}a & b \\ c & d\end{matrix}\right) \in \R^{2\times 2}$, where
\begin{eqnarray}
\begin{split}\label{eq:abcd}
a & = \rho\chi e^{-\ell \chi}+(1-\rho)\chi e^{-\chi}, &  \qquad b & = - (1-\rho)\ell \chi^2e^{-(1-\ell)\chi},\\
c & = -\rho(1-\ell)\chi^2e^{-\ell \chi}, &  \qquad d & = (1-\rho)\chi e^{-(1-\ell )\chi}+\rho\chi e^{-\chi}.
\end{split}
\end{eqnarray}
\end{definition}

The reason for the definition of $a,b,c,d$ and $A$ is the following proposition, which states that the two-dimensional drift is then given by $(1\pm o(1)) Ax/n$.

\begin{proposition}\label{prop:driftmain}
Let $\chi >0$ and $\rho,\ell\in [0,1]$. Let $x_L \in [\ell n]$ and $x_R \in [(1-\ell)n]$, and $x=(x_L,x_R)$ as column vector. If $x_L+x_R = o(n)$, then
%\begin{align*}
%\E[X^t-X^{t+1} \mid X^t = x] =  \left(\begin{matrix} (1\pm O(\tfrac{x_1+x_2}{n}))\cdot a &\ (1\pm O(\tfrac{x_1+x_2}{n}))\cdot b \\ (1\pm O(\tfrac{x_1+x_2}{n}))\cdot c &\ (1\pm O(\tfrac{x_1+x_2}{n}))\cdot d\end{matrix}\right)\cdot x, %=: (1\pm O((x_1+x_2)/n))Ax. 
%\end{align*}
\begin{align*}
\E[X^t-X^{t+1} \mid X^t = x] =  \frac{1}{n}\left(\begin{matrix} (1\pm o(1))\cdot a &\ (1\pm o(1))\cdot b \\ (1\pm o(1))\cdot c &\ (1\pm o(1))\cdot d\end{matrix}\right)\cdot x =: \frac{1\pm o(1)}{n}Ax, 
\end{align*}
where $a,b,c,d$ and $A$ are given by Definition~\ref{def:driftmain}.
\end{proposition}

We will prove the proposition at the end of this section. First, we give our main result for the \ooea on \TwoLins, which now follows easily from Theorem~\ref{thm:2Ddrift}. We note that, as explained in Remark~\ref{rem:sigma}, the bound $n^{1/2+\Omega(1)}$ in the second case is not tight and could be strengthened.

 \begin{theorem}\label{thm:main_EA}
Let $\chi >0$, $\rho, \ell \in [0,1]$, and consider the \ooea with mutation rate $\chi/n$ on \TwoLins. Let $a,b,c,d$ as in Definition~\ref{def:driftmain}, and let $\gamma_0$ be the unique root of $c\gamma^2 + d\gamma = a\gamma + b$ as in~\eqref{eq:defgamma}. Let $T$ be the hitting time of the optimum.
\begin{enumerate}
\item[(a)] Assume that $a\gamma_0 +b >0$ and that the \ooea is started with $o(n)$ zero-bits. Then $T = O(n\log n)$ with high probability.
\item[(b)] Assume that $a\gamma_0 +b <0$ and that the \ooea is started with $n^{1/2+\Omega(1)}$ zero-bits. Then $T$ is superpolynomial with high probability.
\end{enumerate}
\end{theorem}
\begin{proof}[of Theorem~\ref{thm:main_EA}]
We need to check that Theorem~\ref{thm:2Ddrift} is applicable. For (a), assume that the algorithm starts with $x = o(n)$ zero-bits. We choose a function $\sigma = \sigma(n)$ such that $x = o(\sigma n) = o(n)$. Moreover, we require $\sigma = \omega(\sqrt{\log n/n})$. For concreteness, we may set $\sigma := \max\{n^{-1/4}, \sqrt{x/n}\}$. Then condition A (two-dimensional linear drift) holds by Proposition~\ref{prop:driftmain}. Condition B (tail bound on step size) holds since the number of bit flips per generation satisfies such a tail bound for any constant $\chi$. Thus the claim of (a) follows. For (b), we choose $\sigma:=\sqrt{x}n^{-3/4} = n^{-1/2+\Omega(1)}$. Since $x \ge \sqrt{x} n^{1/4} = \sigma n$, Theorem~\ref{thm:2Ddrift} applies and gives that with high probability $T = e^{n^{\Omega(1)}}$. \qed 
\end{proof}

Although it is not difficult to write down an explicit formula for $\gamma_0$ and for the expression $a\gamma_0+b$, the formula is so complicated that we refrain from giving it here. We suspect that for all $\rho,\ell \in (0,1)$ there is a threshold $\chi_0$ such that $a\gamma_0 +b >0$ for all $\chi<\chi_0$ and $a\gamma_0 +b <0$ for all $\chi>\chi_0$, but we couldn't deduce it easily from the explicit formula. We will show this in Section~\ref{sec:sym} for the symmetric case $\rho=\ell = .5$. For the general case, we instead only give the following, slightly weaker corollary.
\begin{corollary}\label{cor:general_case}
Consider the setting of Theorem~\ref{thm:main_EA}. For all $\rho,\ell \in (0,1)$ there are $\chi_1, \chi_2 >0$ such that $a\gamma_0 +b >0$ for all $\chi<\chi_1$ and $a\gamma_0 +b <0$ for all $\chi>\chi_2$. 
\end{corollary}
\begin{proof}
Fix $\rho,\ell \in (0,1)$. We will study the asymptotics of $\gamma_0$ and $a,b,c,d$ in the limits $\chi \to 0$ and $\chi \to \infty$. We start with $\chi \to 0$, which implies $e^{-\chi}, e^{-\ell \chi}, e^{-(1-\ell) \chi} = 1-o(1)$ . For this limit, we have $a= (1\pm o(1))\chi$, $b=-\Theta(\chi^2)$, $c=-\Theta(\chi^2)$ and $d=(1\pm o(1))\chi$. We write the defining Equation~\eqref{eq:defgamma} as $c\gamma_0^2- b = (a-d)\gamma_0$, or equivalently
\begin{align}\label{eq:cor:general_case}
1 = \frac{(a-d)\gamma_0}{c\gamma_0^2- b}.
\end{align}
%\begin{align}\label{eq:cor:general_case}
%c\gamma_0^2- b = (a-d)\gamma_0.
%\end{align}
To understand the asymptotics of $\gamma_0$ (which is a function of $\chi$), fix a constant $C>0$. Assume for the sake of contradiction that there are arbitrarily small values of $\chi$ for which $\gamma_0 \le C\chi$. Let us examine the right hand side of~\eqref{eq:cor:general_case} as $\chi$ approaches $0$ with such values. The denominator is $-\Theta(\chi^2\gamma_0^2) + \Theta(\chi^2) = -O(\chi^4) + \Theta(\chi^2) = \Theta(\chi^2)$. However, since $a-d = \pm o(\chi)$ the numerator is $\pm o(\chi\gamma_0) = \pm o(\chi^2)$, which contradicts the fact that the fraction is constant. Hence, $\gamma_0 > C\chi$ holds for all sufficiently small values of $\chi$. Since this holds for all $C$, we have shown $\gamma_0 = \omega(\chi)$ for $\chi\to 0$, and consequently $a\gamma_0 = \omega(\chi^2)$. Since $b = -\Theta(\chi^2)$, we obtain  $a\gamma_0 +b > 0$ for sufficiently small $\chi$. This shows the first part of the corollary.

For the second part, we study the limit $\chi \to \infty$. The process is symmetric under exchanging the left and right part of the string, which replaces $(\rho,\ell)$ by $(1-\rho,1-\ell)$. Therefore, we may assume $\ell \ge 1/2$, since the other case is covered by symmetry. This assumption implies $e^{-\ell\chi} = O(e^{-(1-\ell)\chi})$, and we also have $e^{-\chi} = o(e^{-\ell\chi})$, which implies $a = \Theta(\chi e^{-\ell\chi}) = O(d)$. In the strict case $\ell > 1/2$ we even have $a = o(d)$. We write the defining Equation~\eqref{eq:defgamma} as 
\begin{align}\label{eq:cor:general_case_2}
\gamma_0^2 + \frac{d-a}{c} \cdot \gamma_0 = \frac{b}{c}.
\end{align}
We compute the asymptotics of the involved terms for. The right hand side is $b/c = \Theta(e^{(2\ell-1)\chi})$. The fraction on the left is $(d-a)/c = \pm O(d/c) = \pm O(e^{(2\ell-1)\chi}/\chi)$, and we may replace the $\pm O$ by $+\Theta$ if $\ell > 1/2$.  In this latter case, omitting the positive term $\gamma_0^2$ in~\eqref{eq:cor:general_case_2} gives $(d-a)/c \cdot \gamma_0 \le b/c$, which implies $\gamma_0 \le (b/c) /((b-a)/c) = O(\chi)$. 
%the left hand side of~\eqref{eq:cor:general_case_2} is at least $\Omega((d-a)/c \cdot \gamma_0) = \Omega(e^{(2\ell-1)\chi}/\chi \cdot \gamma_0)$. Since the right hand side is $\Theta(e^{(2\ell-1)\chi})$, similarly as for the first part we  can conclude $\gamma_0 = O(\chi)$  
% For the other term, we obtain  %Thus $\gamma_0$ satisfies an equation of the form
%\begin{align}\label{eq:asymptotics2}
%\gamma_0^2 \pm O(e^{(2\ell-1)\chi}/\chi) \cdot \gamma_0= \Theta(e^{(2\ell-1)\chi}),
%\end{align}
%where the $\pm O$ may be replaced by $+\Theta$ if $\ell > 1/2$. 
%If $\ell > 1/2$, similarly as for the first part we fix $C>0$ and assume for the sake of contradiction that $\gamma_0 > C$.
%In that case, we may conclude $\gamma_0 = O(\chi)$ because otherwise the left hand side would be asymptotically larger than the right hand side.
%In that case, we may conclude $\gamma_0 = O(\chi)$ because otherwise the left hand side would be asymptotically larger than the right hand side. 
Consequently, in this case $a\gamma_0 = O(\chi^2e^{-\ell\chi}) = o(|b|)$. So $a\gamma_0 +b < 0$ if $\chi$ is sufficiently large. 

In the remaining case $\ell = 1/2$, the asymptotics simplify to $b/c = \Theta(1)$ and $(d-a)/c = \pm O(1/\chi) = \pm o(1)$, so~\eqref{eq:cor:general_case_2} implies $\gamma_0 = \Theta(1)$. Hence,  $a\gamma_0 = \Theta(\chi e^{-\chi/2})$ and $b = -\Theta(\chi^2e^{-\chi/2})$, so again $a\gamma +b < 0$ if $\chi$ is sufficiently large.\qed
\end{proof}

It remains to prove Proposition~\ref{prop:driftmain}. The proof consists mostly of rather simple calculations, which we keep brief.

\begin{proof}[of Proposition~\ref{prop:driftmain}]
Throughout this proof, we write $\beta_L \coloneqq x_L/n$, $\beta_R \coloneqq x_R/n$, $\beta := \beta_L+\beta_R$, and $\Delta = (\Delta^L,\Delta^R) := \E[X^t-X^{t+1} \mid X^t = x]$. %The indices L and R stand for `left' and `right' and are chosen to reduce potential confusion with the indices of $f_1$ and $f_2$. 
We first consider the case where the random environment $f^t$ is given by $f_1$ in the $(t+1)$-st iteration, which happens with probability $\rho$. We denote the drift in this case by $\Delta^L_1$ and $\Delta^R_1$, and similarly for $f_2$, so that 
\begin{align}\label{eq:del_combine}
\Delta^L = \rho\Delta^L_1 + (1-\rho)\Delta^L_2 \qquad \text{ and } \qquad \Delta^R = \rho\Delta^R_1 + (1-\rho)\Delta^R_2.
\end{align}
% First, we want to bound the drift $\Delta_1^L$ in the first part: \begin{align}
%    \Delta_{1}^L \coloneqq \E \left[X_1^t-X_1^{t+1}\mid X_1^t=x_1 \land X_2^t = x_2 \land \textsc{TwoLin}^{\rho,\frac{1}{2}} = f_1 \right].
%\end{align} 

Denote by $\mathcal{F}_{i,j}$ the event that the algorithm flips $i$ zero-bits and $j$ one-bits in the left part. Observe that the quantity $\Delta_{1}^L$ does not depend on which flips happen in the right part of the search point $x^{t}$: Assume that we are in $\mathcal{F}_{i,j}$. There are three possible cases: (a) $i>j$, (b) $i=j$, and (c) $i<j$. In case (a), the offspring is accepted regardless of which flips happen in the right part. In case (b), the number of zero-bits in the left part of the offspring is the same as $X_L^t$, so the contribution to the drift $\Delta_{1}^L$ is zero. Lastly, in case (c), the offspring is rejected irrespective of which flips happen in the right part. Again, the contribution to $\Delta_{1}^L$ is zero. Therefore, we only need to consider case (a) in our computation of $\Delta_{1}^L$. Moreover, by Lemma~\ref{lem:several_flips}, all cases with $i\geq 2$ combined contribute $O(\beta_L^2)$ to the drift. Finally, $\Pr[\mathcal{F}_{1,0}] = (1\pm o(1))\chi e^{-\ell\chi}\beta_L$ follows from~\cite[Corollary~1.4.6]{doerr2020probabilistic}. By the law of total probability we thus have 
\begin{align}\label{eq:del1L} 
    \Delta_{1}^L %& = \sum_{i=1}^{x_1} \sum_{j=0}^{i-1} \E \left[X_1^t-X_1^{t+1}\mid X_1^t=x_1 \land \mathcal{F}_{i,j}\right] 
%    %\Pr \left[\mathcal{F}_{i,j}\right] + 0\\
    &=\sum_{i=1}^{x_L} \sum_{j=0}^{i-1} (i-j) \Pr \left[\mathcal{F}_{i,j}\right] 
%\end{align}
%The probabilities of the events $\mathcal{F}_{i,j}$ are given by \begin{align} \label{preij}
%    \Pr \left[\mathcal{F}_{i,j}\right] = {x_1 \choose i } {n/2-x_1 \choose j} \left(\frac{c}{n}\right)^{i+j} \left(1-\frac{c}{n}\right)^{n/2-i-j}.
%\end{align}
%In the case $i=1$, $j=0$, this probability is given by $(1\pm o(1))ce^{-c/2}\beta_1$. By Lemma \ref{len_2}, we get\begin{align} 
%    \Delta_{1}^L =   
 %= (1\pm o(1))\chi e^{-\ell \chi}\beta_1 \pm O(\beta_1^2) 
 = (1\pm o(1))\chi e^{-\ell \chi}\beta_L.
\end{align}
%Let us now consider the drift in the second part when $\textsc{TwoLin}^{\rho, \frac{1}{2}}=f_1$:
%\begin{align} 
%    \Delta_{1}^R \coloneqq \E \left[X_2^t-X_2^{t+1}\mid X_1^t=x_1 \land X_2^t=x_2 \land \textsc{TwoLin}^{\rho, \frac{1}{2}} = f_1 \right].
%\end{align}
To compute $\Delta_{1}^R$, we again distinguish the cases (a), (b), and (c) from above. As before, the events $\mathcal{F}_{i,j}$ for $i\ge 2$ together contribute at most a $O(\beta_R^2)$ term. The case (c) contributes zero to the drift $\Delta_{1}^R$, as in this case, the offspring is rejected. Let $\mathcal{A}$ and $\mathcal{B}$ denote the events corresponding to the first and second case, respectively. Similarly as before, we have $\Pr[\mathcal{A}] = (1+o(1))\Pr[\mathcal{F}_{1,0}] = (1\pm o(1))\chi e^{-\ell\chi}\beta_L$ and $\Pr[\mathcal{B}] = (1+o(1))\Pr[\mathcal{F}_{0,0}] = (1\pm o(1)) e^{-\ell\chi}$. In case (a), the offspring is always accepted. Since $(1-\ell) n$ bits in the right part of the string are flipped with probability $\chi/n$, and only a $o(1)$ fraction of them is zero-bits, the expected change in that part is thus $-(1\pm o(1))\chi(1-\ell)$ in case (a). In case (b), the calculation is analogous to~\eqref{eq:del1L}, and we get an expected change of $(1\pm o(1))\chi e^{-(1-\ell) \chi}\beta_R$ in the right part of the string. Putting the cases together we obtain
\begin{align}\label{eq:del1R}
\begin{split}
    \Delta_{1}^R & = (1\pm o(1))\chi e^{-\ell \chi}\beta_L \cdot (-\chi(1-\ell)) + (1\pm o(1)) e^{-\ell\chi} \cdot \chi e^{-(1-\ell)\chi}\beta_R \\
    & = - (1\pm o(1))(1-\ell)\chi^2 e^{-\ell \chi}\beta_L + (1\pm o(1)) \chi e^{-\chi}\beta_R.
\end{split}
\end{align}
Note that the first term will contribute to entry $c$ of the matrix, because it depends on $\beta_L = x_L/n$ (first column of the matrix) and contributes to the change of the right part of the string (second row of the matrix). Likewise, the second term contributes to $d$.

By symmetry of $f_1$ and $f_2$, we get analogous terms for the drift in the left and right part in the case that $f^t = f_2$ is chosen, except that the following roles are reversed: the indices $L$ and $R$; the densities $\beta_L$ and $\beta_R$; and the lengths $\ell$ and $(1-\ell)$: 
\begin{align}\label{eq:del2}
\begin{split}
    \Delta^R_{2} &= (1\pm o (1)) \chi e^{-(1-\ell)\chi} \beta_R,\\
    \Delta^L_{2} &= -(1 \pm o(1) ) \ell\chi^2e^{-(1-\ell)\chi} \beta_R + (1\pm o(1))\chi e^{-\chi}\beta_L.
\end{split}
\end{align}
Now we just need to plug~\eqref{eq:del1L},~\eqref{eq:del1R} and~\eqref{eq:del2} into~\eqref{eq:del_combine}, and obtain $\Delta^L = (1\pm o(1))a\beta_L + (1\pm o(1))b\beta_R$ and $\Delta^R = (1\pm o(1))c\beta_L + (1\pm o(1))d\beta_R$, as required. 
%For later reference, we note that the same result comes out if we ignore all events in which at least two zero-bits are flipped.
\qed
%
%Using the law of total expectation, we have \begin{align}
%    \Delta^L & = \Delta^L_{1}\cdot \rho +   \Delta^L_{2} \cdot (1-\rho) 
%     =  (1\pm o(1))\chi \left(e^{-\chi/2}\cdot \rho+ e^{-c}\cdot (1-\rho)\right)\beta_1 \\ & \qquad  - (1 \pm o(1) ) \frac{c^2}{2}e^{-c/2}\cdot (1-\rho ) \beta_2\pm O (\beta_1^2) \pm O(\beta_2^2).
%\end{align}
%and similarly for $\Delta^R$. We note that as $y=o(n)$, we have that $O(\beta_1^2)=o(\beta_1)$ and $O(\beta_2^2)=o(\beta_2)$, concluding the proof.
%
\end{proof}
%\begin{thebibliography}{10}
%\providecommand{\url}[1]{\texttt{#1}}
%\providecommand{\urlprefix}{URL }
%\providecommand{\doi}[1]{https://doi.org/#1}
%\end{thebibliography}

\subsection{Dominiation and the Symmetric Case $\rho=\ell=1/2$}\label{sec:sym}

In this section, we will give the Domination Lemma~\ref{lem:comparison}, and we use it to study the symmetric case $\rho=\ell = 1/2$ in more detail. 
%We abbreviate $\TwoLin := \TwoLin^{.5,.5}$ in this subsection.
We will show two things for the \ooea on $\TwoLin^{.5,.5}$ beyond the general statement in Theorem~\ref{thm:main_EA}: for the positive result, we remove the condition that the algorithm must start with $o(n)$ zero-bits; and we show that there is a threshold $\chi_0\approx 2.557$ such that the algorithm is efficient for $\chi < \chi_0$ and inefficient for $\chi > \chi_0$. We start by inspecting the threshold condition $a\gamma_0+b >0$.
\begin{lemma}\label{lem:symmetric}
Let $a,b,c,d$ be as in Definition~\ref{def:driftmain} and $\gamma_0$ as in Lemma~\ref{lem:defgamma}. For $\rho = \ell = 1/2$, we have
\begin{align}
\label{eq:abcdsym}
a = d = \tfrac{\chi}{2}\big(e^{-\chi/2}+e^{-\chi}\big),  \qquad b = c =- \tfrac{\chi^2}{4}e^{-\chi/2}, \qquad \gamma_0 = 1.
\end{align}
Let $\chi_0 \approx 2.557$ be the unique positive root of $2-\chi +2e^{-\chi/2} =0$. Then $a\gamma_0+b > 0$ for $\chi <\chi_0$ and $a\gamma_0+b < 0$ for $\chi >\chi_0$.
\end{lemma}
\begin{proof}
The formulas for $a,b,c,d$ are simply obtained by plugging $\rho = \ell = 1/2$ into~\eqref{eq:abcd}. Since $a=d$, the defining Equation~\eqref{eq:defgamma} for $\gamma_0$ simplifies to $c\gamma_0^2 = b$, which implies $\gamma_0=1$. For the critical expression $a\gamma_0+b$ we obtain
\begin{align}
a\gamma_0 + b = \tfrac{\chi}{2}\big(e^{-\chi/2}+e^{-\chi}\big) - \tfrac{\chi^2}{4}e^{-\chi/2} = \tfrac{\chi e^{-\chi/2} }{4}\left(2-\chi + 2e^{-\chi/2}\right).
\end{align}
The expression in the bracket is decreasing in $\chi$, so it is positive for $\chi <\chi_0$ and negative for $\chi > \chi_0$.\qed
\end{proof}
Now we are ready to prove a stronger version of Theorem~\ref{thm:main_EA} for the case $\rho=\ell=.5$. The main difference is that the threshold is explicit and that we may assume that the algorithm starts with an arbitrary search point. Finally, the results also hold in expectation.
\begin{theorem}\label{thm:main_EA_sym}
Let $\chi >0$, $\rho = \ell =.5$, and consider the \ooea with mutation rate $\chi/n$ on $\TwoLin^{.5,.5}$, with uniformly random starting point. Let $\chi_0 \approx 2.557$ be the unique root of $2-\chi +2e^{-\chi/2} =0$. Let $T$ be the hitting time of the optimum.
\begin{enumerate}
\item[(a)] If $\chi < \chi_0$, then $T = O(n\log n)$ in expectation and with high probability.
\item[(b)] If $\chi > \chi_0$, then $T$ is superpolynomial in expectation and with high probability.
\end{enumerate}
\end{theorem}
\begin{proof}
The negative statement (b) follows immediately from Theorem~\ref{thm:main_EA} and Lemma~\ref{lem:symmetric}, where the former is applicable since whp a uniformly random starting point has $n^{1/2+\Omega(1)}$ zero-bits. Note that if the runtime is large whp, then it is also large in expectation. So it remains to show (a). As before, we denote by $X_L^t$ and $X_R^t$ the number of zero-bits in the left and right part of the $t$-th search point $x^{t}$, respectively. But now, instead of studying the two-dimensional vector $(X_L^t,X_R^t)$, we will study the one-dimensional potential $X_L^t+X_R^t$, which is simply the number of zero-bits in the string. Since we will now work with a one-dimensional potential, we will slightly adapt our naming conventions for the rest of this proof: we will mark all two-dimensional vectors explicitly with a vector symbol ($\vec X, \vec \Delta, \ldots$), while we use symbols without vector mark for the component sum of the corresponding vectors. So we write $\vec X^t =(X_L,X_R)$ and $X^t = X^t_L + X^t_R$; $\vec \Delta = (\Delta^L,\Delta^R)$ and $\Delta = \Delta^L + \Delta^R$, and so on. In this way, our potential function is denoted by $X^t$, and its drift is denoted by $\Delta$. To avoid confusion, we repeat and extend our notation. For $\vec x = (x_L, x_R) \in [n/2] \times [n/2]$, we let
\begin{align}\label{eq:sym_def1}
    \Delta = \Delta(\vec x) \coloneqq \E \left[ X^t- X^{t+1} \mid \vec X^t = \vec x \right] = \E \left[ X^t- X^{t+1} \mid X_L^t = x_L \land X_R^t=x_R\right].
\end{align}
%where we usually omit the index $t$ on $\Delta(\vec x)$.  
We define $\Delta_k(\vec x)$ to be the drift conditional on flipping exactly $k$ zero-bits
\begin{align}\label{eq:sym_def2}
    \Delta_k(\vec x) & \coloneqq \E \Big[X^t- X^{t+1}\mid \vec X^t = \vec x \land \{k \text{ zero-bits flipped in step $t$}\}\Big].
\end{align}
To refine this notion, let $\mathcal{E}_{i,j}$ denote the event that the algorithm flips $i$ zero-bits in the left part, and $j$ zero-bits in the right part, respectively. We set $\Delta_{i,j}(\vec x)$ to be the drift conditional on $\mathcal{E}_{i,j}$:
 \begin{align}\label{eq:sym_def3}
    &\Delta_{i,j}(\vec x) \coloneqq  \E \left[ X^t- X^{t+1}\mid \mathcal{E}_{i,j} \land \vec X^t = \vec x  \right].
\end{align}
%Moreover, as before we set 
%\begin{align}\label{eq:sym_def4}
%    \Delta^L(\vec x) \coloneqq \E \left[ X_L^t- X_L^{t+1}\mid \vec X^t = \vec x \right],
%\end{align}
%and 
%\begin{align}\label{eq:sym_def5}
%    \Delta^R(\vec x) \coloneqq \E \left[ X_R^t- X_R^{t+1}\mid \vec X^t = \vec x \right],
%\end{align}
%and we define $\Delta^L_{i,j}(\vec x)$ and $\Delta^R_{i,j}(\vec x)$ analogously to $\Delta_{i,j}(\vec x)$. 
We will often omit the argument $\vec x$ of the various $\Delta$'s, to make it easier to read. Finally, we denote by $p_{r_1}$ and $q_{r_2}$ the probabilities of flipping $r_1$ and $r_2$ one-bits in the left and right part of the search point $x^{t}$, respectively. 

Essentially, in the proof of Proposition~\ref{prop:driftmain} we have computed for the general case that $\vec \Delta_1(\vec x) = (1\pm o(1))A\vec x$, and used that $\vec \Delta_k(\vec x)$ is negligible for $k \ge 2$. This was feasible since we were studying the regime $X^t = o(n)$. However, now that we want to extend this to arbitrary values of $X^t$, the terms for $k \ge 2$ are no longer negligible. We cope with them in an indirect way. We first show that $\Delta_{1,0} >0$ and $\Delta_{0,1} > 0$, which is rather easy. Then we show the implication: ``if both $\Delta_{1,0} >0$ and $\Delta_{0,1} > 0$, then $\Delta_{i,j} > 0$ for all $i,j \geq 0$ with $i+j > 0$.'' Hence, we obtain the lower bound $\Delta > \Pr[\mathcal{E}_{1,0}]\Delta_{1,0} + \Pr[\mathcal{E}_{0,1}]\Delta_{0,1}$, which is positive. This is reminiscent of a coupling argument, but formally it is an algebraic calculation and not a coupling. We will outsource the proof of that statement into its own lemma, Lemma~\ref{lem:comparison}. For this proof, we start by showing the following bounds.
\begin{align}\label{eq:10and01}
\begin{split}
\Delta_{1,0} &\ge (1-o(1))\tfrac{e^{-\chi}}{4}(2-\chi+2e^{-\chi/2}), \\
\Delta_{0,1} &\ge (1-o(1))\tfrac{e^{-\chi}}{4}(2-\chi+2e^{-\chi/2}).
\end{split}
\end{align}
Note that the lower bound is strictly positive for sufficiently large $n$, due to the assumption $\chi < \chi_0$.

To see the first line of~\eqref{eq:10and01}, we condition on $\mathcal E_{1,0}$, i.e., that the mutation flips one zero-bit in the left and no zero-bit in the right part. We first consider the case $f^t=f_1$, which has probability $1/2$, and its three subcases that the number of flipped one-bits in the left part is \emph{(a)} larger than one; \emph{(b)} exactly one; \emph{(c)} zero. In case (a), the offspring is rejected, and thus this case contributes zero to the drift. In case (b), the offspring is only accepted if no further one-bits are flipped in the right part, so this case also contributes zero. Case (c) happens with probability $(1-\chi/n)^{n/2-x_L} = (1\pm o(1))e^{-(\chi/2 - \chi x_L/n)}$, and in this case the offspring is always accepted. The expected number of one-bit flips in the right part of the string is $\chi/n\cdot (n/2-x_R) \le \chi/2$, so the expected change in the whole string is at least $1-\chi/2$, where the $1$ is the change in the left part. Altogether, the case $f^t=f_1$ contributes at least
\begin{align}\label{eq:symDelta1a}
\tfrac12(1\pm o(1))e^{-(\chi/2 - \chi x_L/n)}(1-\chi/2))%\stackrel{(\star)}{\ge} \tfrac12 (1\pm o(1))e^{-\chi/2- \chi x_L/n}(1-\chi/2),
\end{align} 
to the drift. %The second step $(\star)$ comes from the fact that the function $e^{-z}(1-z)$ is decreasing for $z \in [0,\chi/2] \subseteq [0,2]$. (It has derivative $e^{-z}(z-2) \le 0$.) 
For the other case, $f^t=f_2$, we distinguish the cases that \emph{(a)} a one-bit in the right part is flipped; \emph{(b)} at least one one-bit is flipped in the left part, but none in the right part; \emph{(c)} no one-bit is flipped in the whole string. In case (a), the offspring is rejected. In case (b), the offspring is either rejected or has the same number of zero-bits as the parent. So both (a) and (b) contribute zero to the drift. Finally, case (c) happens with probability $(1-\chi/n)^{n-x_L-x_R} \ge (1\pm o(1))e^{-\chi}$ and implies a change of exactly $1$. So the case $f^t=f_2$ contributes at least $(1-o(1))\tfrac12 e^{-\chi}$. Adding this to \eqref{eq:symDelta1a}, this shows that
\begin{align}\label{eq:symDelta1b}
\begin{split}
\Delta_{1,0} &\ge (1\pm o(1))\tfrac12e^{-\chi/2-\chi x_L/n}(1-\chi/2) + (1\pm o(1))\tfrac12e^{-\chi} \\
&= (1\pm o(1))\tfrac14e^{-\chi/2-\chi x_L/n}\underbrace{(2-\chi + 2e^{-\chi/2 + \chi x_L/n} )}_{>0 \text{ since } \chi < \chi_0} \\
& \ge (1\pm o(1))\tfrac14e^{-\chi}(2-\chi + 2e^{-\chi/2} ).
\end{split}
\end{align}
Note that in the last step we have removed a factor larger than one, which is a decreasing step since the total expression is positive. Thus we have shown the first claim in~\eqref{eq:10and01}, and the second follows by symmetry. 

As mentioned above we will use the statement ``if both $\Delta_{1,0} >0$ and $\Delta_{0,1} > 0$, then $\Delta_{i,j} > 0$ for all $i,j \geq 0$ with $i+j > 0$'', which we prove in the Domination Lemma~\ref{lem:comparison} below. Here we only show how to conclude the proof of Theorem~\ref{thm:main_EA_sym}(b) modulo Lemma~\ref{lem:comparison}. By this lemma and \eqref{eq:10and01}, all $\Delta_{i,j}$ with $i+j > 0$ are positive. Moreover, $\Delta_{0,0}=0$. Thus, by the law of total probability, for all $\vec x = (x_L,x_R)$, %and abbreviating $\delta := \tfrac{e^{-\chi}}{4}\left(2-\chi+2e^{-\chi/2}\right)$,
\begin{align}
\begin{split}
\Delta(\vec x) & = \sum_{i=0}^{x_L}\sum_{j=0}^{x_R}\Pr[\mathcal{E}_{i,j}]\Delta_{i,j} \ge \Pr[\mathcal{E}_{0,1}]\Delta_{0,1} + \Pr[\mathcal{E}_{0,1}]\Delta_{0,1} \\
& \stackrel{\eqref{eq:10and01}}{\ge} (1-o(1))\tfrac{e^{-\chi}}{4}\left(2-\chi+2e^{-\chi/2}\right)\Pr[\mathcal{E}_{0,1} \text{ or } \mathcal{E}_{1,0}]\\
& \ge (1-o(1))\tfrac{e^{-\chi}}{4}\left(2-\chi+2e^{-\chi/2}\right)(x_L+x_R)\tfrac{\chi}{n}e^{-\chi} = \Omega(\tfrac{x_L+x_R}{n}).
\end{split}
\end{align}
Theorem~\ref{thm:main_EA_sym}(b) now follows from the Multiplicative Drift Theorem~\cite{lengler2020drift}. \qed
\end{proof}

It remains to show the Domination lemma. Interestingly, this lemma holds in larger generality. We prove it for arbitrary $\ell$ and $\rho$. However, this does not suffice to generalize Theorem~\ref{thm:main_EA_sym} to arbitrary $\rho$ and $\ell$. The point where it breaks is that it is generally not true that $\Delta_{0,1}$ and $\Delta_{1,0}$ are both positive if $\Delta_1$ is positive.  
\begin{lemma}[Domination Lemma]\label{lem:comparison}
Let $\rho,\ell\in [0,1]$, $x_L \in [\ell n]$ and $x_R  \in [(1-\ell)n]$. With the notation from~\eqref{eq:sym_def1}--~\eqref{eq:sym_def3}, if $\Delta_{0,1}(x_L,x_R)>0$ and $\Delta_{1,0}(x_L,x_R)>0$ then $\Delta_{i,j}(x_L,x_R) > 0$ for all $i,j \geq 0$ with $i+j > 0$.
\end{lemma}
\begin{proof}
We split the proof into several claims.
\begin{myclaim} We may write $\Delta_{i,j}$ as \label{cldelij}
\begin{align}
\Delta_{i,j}(x_L,x_R) & = \sum_{r_1=0}^{i} \sum_{r_2=0}^{j} (i+j-r_1-r_2)p_{r_1}\cdot q_{r_2}
\\ & \quad + \sum_{r_1= 0}^{i-1}\sum_{r_2=j+1}^{(1-\ell) n-x_R}(i+j-r_1-r_2)p_{r_1}\cdot q_{r_2} \cdot \rho
\\ & \quad + \sum_{r_1=i+1}^{\ell n-x_L} \sum_{r_2=0}^{j-1} (i+j-r_1-r_2)p_{r_1}\cdot q_{r_2}\cdot (1-\rho).
\end{align}
\end{myclaim}

\begin{proof} 
We denote by $\mathcal{A}^{r_1,r_2}$ the event that the algorithm flips $r_1$ one-bits in the first part and $r_2$ one-bits in the second part. Additionally, we define \begin{align}
    \Delta_{i,j}^{r_1,r_2}(x_L,x_R) \coloneqq \E \left[X^t- X^{t+1}\mid \mathcal{A}^{r_1,r_2} \land\mathcal{E}_{i,j} \land X^t_L=x_L \land X_R^t= x_R\right].
\end{align}
By the law of total expectation, we have 
\begin{align}
    \Delta_{i,j}(x_L,x_R)& = \sum_{r_1=0}^{\ell n-x_L}\sum_{r_2=0}^{(1-\ell)n-x_R} \Delta_{i,j}^{r_1,r_2}\cdot \Pr \left[\mathcal{A}^{r_1,r_2}\right]. \label{delij}
\end{align}
Note that the drift is $i+j-r_1-r_2$ in case the offspring $y^t$ is accepted and zero otherwise. We have from the law of total probability
\begin{align}\label{aaa1}
\begin{split}
    \Delta_{i,j}^{r_1,r_2}  = (i+j - r_1-r_2)  \cdot &\Big(\rho \cdot  \Pr\left[x^{t+1}= y^{t} \mid f^t=f_1\right]   \\& \quad +(1-\rho)\cdot \Pr\left[x^{t+1}= y^{t} \mid f^t=f_2\right]   \Big). 
\end{split}
\end{align}
We denote by $\mathbbm{1}_{\{\mathcal{E}\}}$ the indicator random variable of an event $\mathcal{E}$. It follows from the definition of $f_1$ and $f_2$ and the independence of the bit-flips in the first and second part that\begin{align}
     \Pr \left[x^{t+1}= y^{t} \mid f^t=f_1\right] & = \mathbbm{1}_{\{i>r_1\}}+\mathbbm{1}_{\{i=r_1\}} \mathbbm{1}_{\{j>r_2\}} +\tfrac{1}{2}\cdot\mathbbm{1}_{\{i=r_1\}}  \mathbbm{1}_{\{j=r_2\}}. \label{aaa2}
\end{align}
By an analogous argument, we get \begin{align}
    \Pr\left[x^{t+1}= y^{t} \mid f^t=f_2\right]&  =\mathbbm{1}_{\{j>r_2\}}+\mathbbm{1}_{\{j=r_2\}} \mathbbm{1}_{\{i>r_1\}}+\tfrac{1}{2}\cdot\mathbbm{1}_{\{i=r_1\}}   \mathbbm{1}_{\{j=r_2\}}. \label{aaa3}
\end{align}
We consider the following cases:
\begin{multicols}{2}
\begin{enumerate}[(a)]
\item $i\geq r_1$ and $j\geq r_2$,
\item $i> r_1$ and $j<r_2$, 
\item $i= r_1$ and $j<r_2$,
\item $i<r_1$ and $j> r_2$, 
\item $i<r_1$ and $j=r_2$, and
\item $i<r_1$ and $j<r_2$. 
\end{enumerate}
\end{multicols}
For case (a), we get from \eqref{aaa1}, \eqref{aaa2} and \eqref{aaa3}
\begin{align}
    \Delta_{i,j}^{r_1,r_2}  = (i+j-r_1-r_2)\cdot 1,
\end{align}
as if $i=r_1$ and $j=r_2$, then $i+j-r_1-r_2=0$, so we can take the factor $1$ instead of $\frac{1}{2}$ in that case. \\
In case (b) the offspring is accepted if and only if $f^t=f_1$. In accordance with this, we have from \eqref{aaa2} and \eqref{aaa3},
\begin{align}
    \Delta_{i,j}^{r_1,r_2}  = (i+j-r_1-r_2)\cdot \rho.
\end{align}
Similarly, for case (d), we have
\begin{align}
    \Delta_{i,j}^{r_1,r_2}  = (i+j-r_1-r_2)\cdot (1-\rho).
\end{align}
Lastly, we note that in the cases (c), (e), and (f), the offspring is always rejected, so \begin{align}
    \Delta_{i,j}^{r_1,r_2} =0.
\end{align}
Since the algorithm flips bits independently, we have $\Pr[\mathcal{A}^{r_1,r_2}]=p_{r_1}\cdot q_{r_2}$. Splitting the double sum in \eqref{delij} according to the above cases yields
\begin{align}
     \Delta_{i,j}(x_L,x_R)& =  \underbrace{\sum_{r_1=0}^{i} \sum_{r_2=0}^{j} (i+j-r_1-r_2)p_{r_1}\cdot q_{r_2}}_{\text{case (a).}}\\ 
     & + \underbrace{\sum_{r_1= 0}^{i-1}\sum_{r_2=j+1}^{(1-\ell)n-x_R} (i+j-r_1-r_2)\rho\cdot p_{r_1}\cdot q_{r_2}}_{\text{case (b).}}
     \\ 
     & + \underbrace{\sum_{r_1=i+1}^{\ell n-x_L} \sum_{r_2=0}^{j-1} (i+j-r_1-r_2)(1-\rho)\cdot p_{r_1}\cdot q_{r_2}}_{\text{case (d).}}, 
\end{align}
as desired.
\end{proof}

\begin{myclaim} \label{cl_ci}
There is $c_{i,x_L}> 0$ such that for all $1\leq i\leq x_L$ and all $x_R \in[(1-\ell)n]$,\begin{align}
    \Delta_{i,0}(x_L,x_R)\geq c_{i,x_L}\Delta_{1,0}(x_L,x_R).
\end{align}
Analogously, there is $d_{j,x_R}> 0 $ such that for all $1\leq j \leq x_R$ and all $x_L \in [\ell n]$,\begin{align}
    \Delta_{0,j}(x_L,x_R)\geq d_{j,x_R}\Delta_{0,1}(x_L,x_R).
\end{align}
\end{myclaim}

\begin{proof}
First, we apply Claim \ref{cldelij} to $\Delta_{1,0}$:
\begin{align}
    \Delta_{1,0}& = \sum_{r_1=0}^1 (1-r_1)p_{r_1}\cdot q_0 + \sum_{r_2=1}^{(1-\ell)n-x_R}(1-r_2)p_{0}\cdot q_{r_2}\cdot \rho\\
    & = p_0\cdot q_0 +  \sum_{r_2=1}^{(1-\ell)n-x_R}(1-r_2)p_{0}\cdot q_{r_2}\cdot \rho. \label{26_1}
\end{align}
We define $c_{i,x_L}$ as follows: \begin{align}
    c_{i,x_L} \coloneqq \frac{\sum_{r_1=0}^{i-1}p_{r_1}}{p_0}\geq 1 >0.
\end{align}
From \eqref{26_1} we have\begin{align}
     c_{i,x_L}  \Delta_{1,0} = q_0 \left(\sum_{r_1=0}^{i-1}p_{r_1} \right) + \left(\sum_{r_2=1}^{(1-\ell)n-x_R}(1-r_2)q_{r_2}\cdot \rho\right)\left(\sum_{r_1=0}^{i-1}p_{r_1} \right)  . \label{26_2}
\end{align}
Now, we apply Claim \ref{cldelij} to $\Delta_{i,0}$:
\begin{align}
     \Delta_{i,0}& = \sum_{r_1=0}^{i-1} (i-r_1) p_{r_1}\cdot q_0 + \sum_{r_1= 0}^{i-1}\sum_{r_2=1}^{(1-\ell)n-x_R}(i-r_1-r_2)p_{r_1}\cdot q_{r_2}\cdot \rho. \label{26_3}
\end{align}
Note that $i-r_1\geq 1$, as $r_1<i$. So we get from \eqref{26_3} and the fact that probabilities are non-negative\begin{align}
     \Delta_{i,0}&  \geq \sum_{r_1=0}^{i-1} p_{r_1}\cdot q_0 
     + \sum_{r_1=0}^{i-1}\sum_{r_2=1}^{(1-\ell)n-x_R}(1-r_2)p_{r_1}\cdot q_{r_2}\cdot \rho \\
     & =q_0 \sum_{r_1=0}^{i-1}  p_{r_1}+ \sum_{r_1= 0}^{i-1}p_{r_1}\sum_{r_2=1}^{(1-\ell)n-x_R}(1-r_2) q_{r_2} \cdot \rho.\label{26_4}
\end{align}
Since the term $\sum_{r_2=1}^{(1-\ell)n-x_R}(1-r_2) q_{r_2}\cdot \rho$ above does not depend on $r_1$, we can factor it out of the sum. Hence, we get with \eqref{26_2} and \eqref{26_4}
\begin{align}
    \Delta_{i,0} \geq c_{i,x_L}  \Delta_{1,0},
\end{align}
yielding the first part of the claim. If we set \begin{align}
    d_{j,x_R} \coloneqq \frac{\sum_{r_2=0}^{j-1}q_{r_2}}{q_0}\geq 1 >0,
\end{align}
the second part of the claim can be shown analogously. 
\end{proof}

\begin{myclaim} \label{laange}
For all  $0\leq i \leq x_L$ and $0\leq j \leq x_R$, we have \begin{align}
    \Delta_{i,j}(x_L,x_R)\geq \Delta_{i,0}(x_L,x_R)+\Delta_{0,j}(x_L,x_R).    
\end{align}
\end{myclaim}

\begin{proof}
We will show\begin{align}
    \Delta_{i,j}-\Delta_{i,0}-\Delta_{0,j}\geq 0. \label{to_show}
\end{align}
From Claim \ref{cldelij}, we get \begin{align}
     \Delta_{i,j}& = \underbrace{\sum_{r_1=0}^{i} \sum_{r_2=0}^{j} (i+j-r_1-r_2)p_{r_1}\cdot q_{r_2}}_{\eqqcolon \ A_0}
     \\ & \quad + \underbrace{\sum_{r_1= 0}^{i-1}\sum_{r_2=j+1}^{(1-\ell)n-x_R}(i+j-r_1-r_2)p_{r_1}\cdot q_{r_2}\cdot \rho}_{\eqqcolon\ B_0} 
     \\ & \quad + \underbrace{\sum_{r_1=i+1}^{\ell n-x_L} \sum_{r_2=0}^{j-1} (i+j-r_1-r_2)p_{r_1}\cdot q_{r_2}\cdot (1-\rho)}_{\eqqcolon\ C_0}= A_0+B_0+C_0.
\end{align}
We rewrite $A_0$ by writing the sums for $r_1=i$ and $r_2=j$ separately, and noting that the term for $r_1=i$ and $r_2=j$ is zero, so it can be neglected:\begin{align}
    A_0 & = \sum_{r_1=0}^{i} \sum_{r_2=0}^{j} (i+j-r_1-r_2)p_{r_1}\cdot q_{r_2} \\
    & = \underbrace{\sum_{r_1=0}^{i-1} \sum_{r_2=0}^{j-1}  (i+j-r_1-r_2)p_{r_1}\cdot q_{r_2}}_{\eqqcolon \ A_0^0} 
    +  \underbrace{\sum_{r_2=0}^{j-1}  (j-r_2)p_{i}\cdot q_{r_2}}_{\eqqcolon \ A_0^5}\\ & \qquad + \underbrace{\sum_{r_1=0}^{i-1} (i-r_1)p_{r_1}\cdot q_{j}}_{\eqqcolon \ A_0^6}= A_0^0+ A_0^5+A_0^6.
\end{align}
Now, we rewrite $A_0^0$. First, we multiply $A_0^0$ with the factor $(\rho + (1-\rho))=1$, then we use distributivity.
\begin{align}
    A_0^0 & = \left( \sum_{r_1=0}^{i-1} \sum_{r_2=0}^{j-1} (i+j-r_1-r_2)p_{r_1}\cdot q_{r_2}\right) \cdot \left(\rho + (1-\rho)\right)\\
    &= \sum_{r_1=0}^{i-1} \sum_{r_2=0}^{j-1} (i+j-r_1-r_2)p_{r_1}\cdot q_{r_2} \cdot \rho
    \\ &  + \sum_{r_1=0}^{i-1} \sum_{r_2=0}^{j-1} (i+j-r_1-r_2)p_{r_1}\cdot q_{r_2}\cdot (1-\rho).
\end{align}
Next, we separate the sum for $r_2= 0$ in the first double sum and the sum for $r_1=0$ in the second double sum:
\begin{align}
    A_0^0& = 
    \underbrace{\sum_{r_1=0}^{i-1} \sum_{r_2=1}^{j-1} (i+j-r_1-r_2)p_{r_1}\cdot q_{r_2}\cdot \rho}_{\eqqcolon \ A_0^1} + 
    \underbrace{\sum_{r_1=0}^{i-1}(i+j-r_1)p_{r_1}\cdot q_{0}\cdot \rho}_{\eqqcolon \ A_0^2} 
    \\ & \qquad  +
    \underbrace{\sum_{r_1=1}^{i-1} \sum_{r_2=0}^{j-1} (i+j-r_1-r_2)p_{r_1}\cdot q_{r_2}\cdot (1-\rho)}_{\eqqcolon \ A_0^3}
    \\ & \qquad  +
    \underbrace{\sum_{r_2=0}^{j-1}(i+j-r_2)p_{0}\cdot q_{r_2}\cdot (1-\rho)}_{\eqqcolon\ A_0^4}.
\end{align}
So we have $A_0 = A_0^1+A_0^2+A_0^3+A_0^4+A_0^5+A_0^6$. Next, we apply Claim \ref{cldelij} to $\Delta_{i,0}$:
\begin{align}
     \Delta_{i,0} = \underbrace{\sum_{r_1=0}^{i-1}  (i-r_1) p_{r_1}\cdot q_{0}}_{\eqqcolon\ A_1} + \underbrace{\sum_{r_1=0}^{i-1} \sum_{r_2=1}^{(1-\ell)n-x_R} (i-r_1-r_2)p_{r_1}\cdot q_{r_2}\cdot \rho}_{\eqqcolon\ B_1}= A_1+B_1,
\end{align}
where we note that in the first sum, the term for $r_1=i$ is equal to zero, so it can be omitted. We rewrite $B_1$ by splitting it into the cases $r_2<j$, $r_2=j$, and $r_2>j$. \begin{align}
    B_1&  = \underbrace{\sum_{r_1=0}^{i-1} \sum_{r_2=1}^{j-1} (i-r_1-r_2)p_{r_1}\cdot q_{r_2}\cdot \rho}_{\eqqcolon\ B_1^1} 
    + \underbrace{\sum_{r_1=0}^{i-1}(i-r_1-j)p_{r_1}\cdot q_{j}\cdot \rho}_{\eqqcolon\ B_1^2}\\ 
    & \qquad + \underbrace{\sum_{r_1=0}^{i-1} \sum_{r_2=j+1}^{(1-\ell)n-x_R} (i-r_1-r_2)p_{r_1}\cdot q_{r_2}\cdot \rho}_{\eqqcolon\ B_1^3}= B_1^1+B_1^2+B_1^3. \label{b_1}
\end{align}
Finally, we apply Claim \ref{cldelij} to $\Delta_{0,j}$. 
\begin{align}
     \Delta_{0,j}= \underbrace{\sum_{r_2=0}^{j-1}  (j-r_2) p_{0}\cdot q_{r_2}}_{\eqqcolon\ A_2} + \underbrace{\sum_{r_1=1}^{\ell n-x_L} \sum_{r_2=0}^{j-1} (j-r_1-r_2)p_{r_1}\cdot q_{r_2}\cdot (1-\rho)}_{\eqqcolon\ C_2}= A_2+C_2.
\end{align}
We rewrite $C_2$ in analogy to \eqref{b_1}, i.e., splitting into the cases $r_1<i$, $r_1=i$, and $r_1>i$.
\begin{align}
    C_2 & = \underbrace{\sum_{r_1=1}^{i-1} \sum_{r_2=0}^{j-1} (j-r_1-r_2)p_{r_1}\cdot q_{r_2}\cdot (1-\rho)}_{\eqqcolon\ C_2^1} 
    +\underbrace{\sum_{r_2=0}^{j-1} (j-i-r_2)p_{i}\cdot q_{r_2}\cdot (1-\rho)}_{\eqqcolon\ C_2^2} \\ 
    & \qquad +\underbrace{\sum_{r_1=i+1}^{\ell n-x_L} \sum_{r_2=0}^{j-1} (j-r_1-r_2)p_{r_1}\cdot q_{r_2}\cdot (1-\rho)}_{\eqqcolon\ C_2^3}
    = C_2^1+C_2^2+C_2^3. 
\end{align}
With the above preparations, \eqref{to_show} is equivalent to 
\begin{align}
\begin{split}
    & A_0+B_0+C_0-(A_1+B_1)-(A_2+C_2)= 
    A_0^1+A_0^2+A_0^3+A_0^4+A_0^5+A_0^6 \\ & \qquad +B_0+C_0 -A_1-B_1^1-B_1^2-B_1^3-A_2-C_2^1-C_2^2-C_2^3\geq 0.
\end{split}
\end{align}
Reordering the summands, we need
\begin{align}
\begin{split}
    &(A_0^1 - B_1^1)+ (A_0^2- A_1)+ (A_0^3 -C_2^1)+ (A_0^4 -A_2)\\
    & \qquad + (A_0^5 -C_2^2)+ (A_0^6-B_1^2 )+ (B_0-B_1^3)+ (C_0-C_2^3) \geq 0. \label{to_show2}
\end{split}
\end{align}
We compute \begin{align}
    A_0^5 -C_2^2 & = \sum_{r_2=0}^{j-1}  (j-r_2)p_{i}\cdot q_{r_2} - \sum_{r_2=0}^{j-1} (j-i-r_2)p_{i}\cdot q_{r_2}\cdot (1-\rho) \\
    & = \sum_{r_2=0}^{j-1} \left((1-\rho)i +\rho (j-r_2) \right)p_{i}\cdot q_{r_2} \geq 0,
\end{align}
as $\rho \leq 1$ and $r_2<j$. Similarly, \begin{align}
    A_0^6-B_1^2 = \sum_{r_1=0}^{i-1} ((1-\rho)(i-r_1)+ \rho j) p_{r_1}\cdot q_{j} \geq 0.
\end{align}
Furthermore, \begin{align}
    B_0-B_1^3 = \sum_{r_1= 0}^{i-1}\sum_{r_2=j+1}^{(1-\ell)n-x_R}j\cdot p_{r_1}\cdot q_{r_2} \cdot \rho \geq 0,
\end{align}
and \begin{align}
    C_0-C_2^3 & =\sum_{r_1=i+1}^{\ell n-x_L} \sum_{r_2=0}^{j-1} i\cdot p_{r_1}\cdot q_{r_2} \cdot (1-\rho) \geq 0.
\end{align}
Hence, to show \eqref{to_show2}, it suffices to show \begin{align}
    &(A_0^1 - B_1^1)+ (A_0^2- A_1)+ (A_0^3 -C_2^1)+ (A_0^4 -A_2)\geq 0. \label{sufficient}
\end{align}

Next, we compute \begin{align}
    A_0^1-B_1^1&  = \sum_{r_1=0}^{i-1} \sum_{r_2=1}^{j-1} j \cdot p_{r_1}\cdot q_{r_2} \cdot \rho, \\
    A_0^3 -C_2^1 & =\sum_{r_1=1}^{i-1} \sum_{r_2=0}^{j-1} i\cdot p_{r_1}\cdot q_{r_2}\cdot (1-\rho).
\end{align}
Additionally, we have\begin{align}
    A_0^2- A_1 & = \sum_{r_1=0}^{i-1} (i+j-r_1) p_{r_1}\cdot q_{0} \cdot \rho \ - \sum_{r_1=0}^{i-1} (i-r_1) p_{r_1}\cdot q_{0} \\
    & = \sum_{r_1=0}^{i-1} (\rho\cdot j - (1-\rho)i +  (1-\rho) r_1) p_{r_1}\cdot q_{0}\\ & \geq \sum_{r_1=0}^{i-1} (\rho\cdot j - (1-\rho)i)  p_{r_1}\cdot q_{0},
\end{align}
and similarly \begin{align}
    A_0^4 -A_2 & =\sum_{r_2=0}^{j-1}(i+j-r_2)p_{0}\cdot q_{r_2} \cdot (1-\rho)- \sum_{r_2=0}^{j-1}(j-r_2) p_{0}\cdot q_{r_2} \\
    & \geq \sum_{r_2=0}^{j-1}((1-\rho)i-\rho j)p_{0}\cdot q_{r_2}.
\end{align}
Furthermore, \begin{align}
    (A_0^1-B_1^1)+ (A_0^4 -A_2 ) & \geq  \sum_{r_2=1}^{j-1}j\cdot p_0\cdot q_{r_2} \cdot \rho
    +  \sum_{r_1=1}^{i-1} \sum_{r_2=1}^{j-1} j \cdot p_{r_1}\cdot q_{r_2} \cdot \rho
    \\& \quad  + ((1-\rho)i-\rho j)p_{0}\cdot q_{0}
    \\ & \quad + \sum_{r_2=1}^{j-1}((1-\rho)i-\rho j)p_{0}\cdot q_{r_2} 
    \\ & \geq ((1-\rho)i-\rho j)p_0\cdot q_0 .\label{a_0_1}
\end{align}
Similarly,
\begin{align}
    (A_0^2- A_1) + (A_0^3 -C_2^1) & \geq (\rho j-(1-\rho)i)p_0\cdot q_0  + \sum_{r_1=1}^{i-1} (\rho\cdot j - (1-\rho)i)  p_{r_1}\cdot q_{0}\\
    & + \sum_{r_1=1}^{i-1}  i\cdot p_{r_1}\cdot q_{0}\cdot (1-\rho)\\
    & +\sum_{r_1=1}^{i-1} \sum_{r_2=1}^{j-1} i\cdot p_{r_1}\cdot q_{r_2}\cdot (1-\rho)\\
    &\geq (\rho j-(1-\rho)i)p_0\cdot q_0. \label{a_0_2}
\end{align}
Combining \eqref{a_0_1} and \eqref{a_0_2}, yields
\begin{align}
\begin{split}
    (A_0^1 - B_1^1)+ (A_0^2- A_1)+ (A_0^3 -C_2^1)+ (A_0^4 -A_2) & \geq \\ ((1-\rho)i-\rho j)p_0\cdot q_0 +(\rho j-(1-\rho)i)p_0\cdot q_0 &  = 0.
\end{split}
\end{align}
yielding \eqref{sufficient}, which concludes the proof of Claim~\ref{laange}, and also the proof of Lemma~\ref{lem:comparison}.
\end{proof}\qed
\end{proof}

\section{Conclusion}
In this paper, we have shown how to handle two-dimensional multiplicative/linear drift, even in the presence of error terms. Naturally, our paper is only a first step in that direction. As we have discussed in Remark~\ref{rem:sigma}, we have only stated the main result, Theorem~\ref{thm:2Ddrift}, for one possible scaling, where the drift is of order $\Theta(1/n)$. While this is one of the most prominent cases, it will be worthwhile to develop the drift theorem further to cover more general settings. Also, the restriction to a sublinear value of $\|X^0\|$ in Theorem~\ref{thm:2Ddrift}(a) is not quite satisfactory, and we hope that it can be removed in future work. But the two most interesting questions clearly are:
\begin{itemize}
\item Can the methods developed here for the two-dimensional case be generalized to arbitrary finite dimensions, where the drift is given by $A\cdot X$ for a $d\times d$ matrix $A$? What conditions does $A$ need to satisfy? Rowe gives a sufficient condition for fast convergence in~\cite{rowe2018linear} (in the absence of error terms), but it remains unclear whether this condition is necessary.
\item Can we develop general methods for the two-dimensional case in which the drift is non-linear? We have given a possible approach in Section~\ref{sec:interpretation}, but currently this is only an idea, and a lot of development is needed to turn it into a general method or framework.
\end{itemize}

The second contribution of our paper is introducing the dynamic \TwoLin benchmark, and analyzing the \ooea on it. We believe that this result is quite interesting, as it shows a failure mode of the \ooea in a minimal example. We believe that the \TwoLin benchmark is interesting in its own right, and it is simple enough to be amenable to theoretical analysis. We are curious to see how other algorithms perform on this benchmark. 

One drawback of \TwoLin is that the two fitnesses $f_1(x)$ and $f_2(x)$ are very different. Another line of research could try to find functions $f_1,f_2$ such that $f_1(x) \approx f_2(x)$, and still the same failure mode happens. For example, it is unclear to us what happens if we replace the weights $1$ and $n$ by the weights $1$ and $2$.

%\bibliographystyle{splncs04}
%\bibliography{../references}

\end{document}